\documentclass{article} 
\usepackage[preprint]{colm2026_conference}

\usepackage[T1]{fontenc}
\usepackage[dvipsnames]{xcolor}
\usepackage{microtype}
\usepackage{url}
\usepackage{booktabs}
\usepackage{graphicx}
\usepackage{subcaption}
\usepackage{multirow}
\usepackage{multicol}
\usepackage{algorithm}
\usepackage{algpseudocode}
\usepackage{amsmath,amssymb}
\usepackage{mathtools}
\usepackage{amsthm}
\usepackage{tikz}
\usepackage{wrapfig}
\usepackage[most]{tcolorbox}
\usepackage{fancyvrb}
\usepackage{hyperref}

\usetikzlibrary{shapes.geometric, arrows.meta, positioning, fit, backgrounds, calc}

\theoremstyle{plain}
\newtheorem{theorem}{Theorem}[section]
\newtheorem{proposition}[theorem]{Proposition}

\theoremstyle{definition}

\theoremstyle{remark}

\usepackage{lineno}

\definecolor{darkblue}{rgb}{0, 0, 0.5}
\hypersetup{colorlinks=true, citecolor=darkblue, linkcolor=darkblue, urlcolor=darkblue}

\newcommand{\name}[1]{\textsc{HeaPA}}

\definecolor{poolblue}{RGB}{66, 133, 244}
\definecolor{lowred}{RGB}{234, 67, 53}
\definecolor{highgreen}{RGB}{52, 168, 83}
\definecolor{sampleorange}{RGB}{251, 188, 5}
\definecolor{augpurple}{RGB}{142, 68, 173}
\definecolor{teacherteal}{RGB}{26, 188, 156}
\definecolor{rewardyellow}{RGB}{241, 196, 15}
\definecolor{policyblue}{RGB}{52, 152, 219}

\definecolor{PromptBg}{RGB}{250,250,250}
\definecolor{PromptFrame}{RGB}{220,220,220}
\definecolor{PromptIn}{RGB}{0,92,185}        
\definecolor{PromptOut}{RGB}{134,64,152}     
\definecolor{PromptTag}{RGB}{0,110,103}      

\tcbset{
  promptbox/.style={
    enhanced,
    colback=PromptBg,
    colframe=PromptFrame,
    boxrule=0.6pt,
    arc=2pt,
    left=8pt,right=8pt,top=6pt,bottom=6pt,
    fonttitle=\bfseries,
    breakable
  }
}

\DefineVerbatimEnvironment{PromptVerbatim}{Verbatim}{
  fontsize=\footnotesize,
  baselinestretch=1,
  commandchars=\\\{\}, 
  xleftmargin=0em,
  xrightmargin=0em
}

\title{\name{}: Difficulty-Aware Heap Sampling and On-Policy Query Augmentation for LLM Reinforcement Learning}

\author{
Weiqi Wang$^{1,2}$\thanks{Work done during his internship at Amazon.com Inc.} \quad Xin Liu$^{1}$ \quad Binxuan Huang$^{1}$ \quad Hejie Cui$^{1}$ \\
~\textbf{Rongzhi Zhang}$^{1}$ \quad \textbf{Changlong Yu}$^{1}$ \quad \textbf{Shuowei Jin}$^{1}$ \quad ~\textbf{Jingfeng Yang}$^{1}$ \\
~\textbf{Qingyu Yin}$^{1}$ \quad \textbf{Zhengyang Wang}$^{1}$ \quad \textbf{Zheng Li}$^{1}$ \quad \textbf{Yifan Gao}$^{1}$ \\
~\textbf{Priyanka Nigam}$^{1}$ \quad \textbf{Bing Yin}$^{1}$ \quad \textbf{Lihong Li}$^{1}$ \quad \textbf{Yangqiu Song}$^{1,2}$ \\
$^{1}$ Stores Foundational AI, Amazon Inc., Palo Alto, CA, USA \\
$^{2}$ Department of Computer Science and Engineering, HKUST, Hong Kong SAR, China \\
\texttt{wwangbw@cse.ust.hk, xliucr@amazon.com}
}

\begin{document}

\ifcolmsubmission
\linenumbers
\fi

\maketitle

\begin{abstract}
RLVR has become a standard recipe for training LLMs on reasoning tasks with verifiable outcomes, but when rollout generation dominates the cost, efficiency hinges on which prompts are sampled and when.
In practice, prompt pools are often static or only weakly coupled to policy progress, so uniform sampling fails to track the moving capability frontier and wastes rollouts on regions that are already solved or still unreachable.
Prior methods improve efficiency via filtering, curricula, adaptive rollout allocation, or teacher guidance, but they often assume a fixed pool, which does not support stable on-policy pool growth, or they introduce additional teacher cost and latency.
In this work, we propose \textbf{\name{}} \textit{(\textbf{\underline{Hea}}p Sampling and On-\textbf{\underline{P}}olicy Query \textbf{\underline{A}}ugmentation)}, which maintains a bounded, evolving pool, tracks the frontier with heap-based boundary sampling, grows the pool via on-policy augmentation under lightweight asynchronous validation, and stabilizes correlated queries via topology-aware pool statistics re-estimation and controlled reinsertion.
Across two training corpora, two training recipes, and seven benchmarks, \name{} consistently improves accuracy and reaches target performance with fewer computations at comparable wall-clock time.
Analyses attribute the gains to frontier-focused sampling and on-policy pool growth, with more pronounced improvements at mid-to-large model scales.
Our training code is publicly available at \href{https://github.com/horizon-llm/HeaPA}{https://github.com/horizon-llm/HeaPA}.
\end{abstract}

\section{Introduction}
Large language models (LLMs;~\citealp{Gemini2.5,KimiK2,Seed1.5}) have made rapid progress on reasoning tasks, in part driven by reinforcement learning with verifiable rewards (RLVR;~\citealp{GRPO,DBLP:journals/corr/abs-2503-06639}), where correctness can be checked automatically (with a small set of discrete verifier outcomes).
A major practical bottleneck lies in how prompts are selected throughout training.
When rollout generation dominates the cost, uniform sampling from a static or weakly coupled prompt pool can miss the moving frontier~\citep{DBLP:journals/corr/abs-2506-02177}.
As a result, batches frequently contain extremes: very hard prompts that yield mostly incorrect or overlong rollouts, and very easy prompts that yield mostly correct rollouts.
Both provide little learning signal for group-based RLVR updates while consuming substantial rollout budget~\citep{DAPO}.

Recent efforts to improve RLVR efficiency follow complementary directions.
A data-centric line improves utilization through selection and filtering, including oversampling trajectories and downsampling to retain informative rollouts, as well as prompt-level scoring and curriculum scheduling over large corpora~\citep{DBLP:journals/corr/abs-2504-13818,DBLP:journals/corr/abs-2509-03403,DBLP:journals/corr/abs-2509-02479,DBLP:journals/corr/abs-2508-17445,chen-etal-2025-scale,DBLP:journals/corr/abs-2506-11480}.
Complementary rollout-side methods reallocate inference budget toward higher-uncertainty, higher-potential prompts~\citep{Reinforce-Ada,DBLP:journals/corr/abs-2507-04632}, while teacher-augmented approaches densify supervision on the student's rollouts via on-policy distillation, including black-box formulations that rely only on teacher text~\citep{lu2025onpolicydistillation,blackboxOPD}.

Although these directions improve RLVR efficiency, several gaps remain.
Many strategies implicitly assume a largely fixed prompt pool and focus on filtering, resampling, or reallocating rollout effort within it, making it difficult to stay aligned with the moving capability frontier as the policy improves.
As the frontier shifts, difficulty and learnability estimates can become stale, and rollouts are once again spent on saturated easy prompts or unreachable hard prompts.
Moreover, query-side selection and curriculum scheduling primarily operate over pre-existing corpora and provide limited mechanisms for growing new frontier-aligned prompts on-policy, while teacher-augmented supervision can shift the bottleneck to teacher cost and latency and may be brittle across model families due to tokenization and generation mismatches~\citep{lu2025onpolicydistillation,blackboxOPD}.
Finally, when new prompts are introduced through augmentation, they are often correlated through shared templates or local edits, and naive reinsertion can destabilize sampling priorities and induce oscillatory curricula.

Motivated by these gaps, we propose \textbf{\name{}} \textit{(\textbf{\underline{Hea}}p Sampling and On-\textbf{\underline{P}}olicy Query \textbf{\underline{A}}ugmentation)}, a query-lifecycle framework that jointly addresses frontier tracking, on-policy pool growth, and stable integration.
\name{} maintains a bounded query pool and concentrates training on the evolving capability frontier via heap-based boundary sampling, which prioritizes a medium-difficulty band where rollouts are neither always correct nor always incorrect.
In parallel, \name{} grows the pool on-policy by letting the current policy generate new queries and routing them through asynchronous, lightweight validation to obtain verifiable answers before insertion, keeping data expansion aligned with the student while controlling overhead.
To prevent correlated augmentations from destabilizing the curriculum, \name{} tracks augmentation lineage and applies topology-aware pool statistics re-estimation via lineage propagation with controlled reinsertion, yielding a smoother coupling between query generation, sampling, and policy updates.

We evaluate \name{} by plugging it into GRPO and DAPO, using Qwen2.5-7B as the backbone, and training on two widely used math corpora, DAPO-Math and OpenR1-Math~\citep{Qwen2.5,GRPO,DAPO,OpenR1}.
Across seven held-out benchmarks spanning competition math and broader reasoning, \name{} consistently improves over the original pipelines and over strong sampling and augmentation baselines, and reaches the same target performance with fewer rollout tokens at comparable wall-clock time.
Our analysis shows that heap-based boundary sampling concentrates probability mass near the frontier rather than on extreme difficulty, that verified on-policy augmentations contribute higher reward and advantage signals than seed queries, and that the benefits are especially pronounced at mid-to-large model scales in parameter scaling studies.
Finally, profiling confirms that pool operations and asynchronous validation introduce only minor overhead relative to rollout generation, preserving practical throughput at scale.

\section{Related Works}
\label{sec:related_works}
\subsection{Data Selection and Sampling in LLM Post-Training}
Prior work studies data selection and sampling as a primary lever for improving the efficiency and stability of LLM reinforcement learning.
In RLVR, where correctness provides an automatic signal, this is often realized through selection at two granularities:
(i) \emph{trajectory-level} oversampling followed by downsampling or filtering to retain informative rollouts, for example by increasing within-group reward diversity~\citep{DBLP:journals/corr/abs-2504-13818}, using process-level signals to reduce spurious correctness~\citep{DBLP:journals/corr/abs-2509-03403}, filtering unstable tool-integrated traces~\citep{DBLP:journals/corr/abs-2509-02479}, or expanding collection via structured branching or search to generate candidates for reuse~\citep{DBLP:journals/corr/abs-2508-17445};
and (ii) \emph{prompt-level} reshaping of the training distribution, either by selecting or scheduling prompts from large corpora or by reallocating rollout compute toward higher-uncertainty, higher-potential prompts to avoid low-signal groups, including corpus-based prompt scoring and curricula~\citep{chen-etal-2025-scale,DBLP:journals/corr/abs-2506-11480}, adaptive rollout allocation~\citep{Reinforce-Ada,DBLP:journals/corr/abs-2507-04632,DBLP:journals/corr/abs-2509-25808}, filtering uninformative prompts~\citep{DAPO,DBLP:journals/corr/abs-2504-11343}, modifying advantage computation~\citep{DBLP:journals/corr/abs-2509-18851,DBLP:journals/corr/abs-2509-21880}, and curriculum-style sampling based on historical statistics or online estimates~\citep{DBLP:journals/corr/abs-2504-05520,DBLP:journals/corr/abs-2506-09016}.
Our approach unifies prompt-level distribution shaping with online pool maintenance: we keep a bounded, evolving query pool coupled to policy progress and apply heap-based boundary sampling to track the empirical frontier via reward trends, focusing probability mass on a medium-difficulty band instead of relying on offline scoring or ad hoc filtering rules.

\paragraph{Relation to PCL and SvS.}
Two closely related directions are prompt curriculum learning (PCL) and self-play with variational synthesis (SvS).
PCL selects intermediate-difficulty prompts using a policy-dependent difficulty estimate, which is closely aligned with our motivation of focusing RLVR on the moving capability frontier.
In contrast, \name{} tracks the frontier through empirical rollout-derived pool statistics and heap boundary sampling, while also supporting on-policy pool growth and lineage-aware reinsertion.
SvS synthesizes answer-preserving variants from prompts the policy can already solve, avoiding teacher annotation by construction.
\name{} is complementary: its lifecycle can incorporate teacher-free answer-preserving variants, but also supports teacher-annotated new-answer augmentations when answer-preserving transformations are insufficient.
We provide full PCL-style and SvS-style comparisons in Appendix~\ref{app:related_baselines}.

\subsection{Teacher-Augmented Training and Distillation}
Another line of work studies teacher-augmented RL training, aiming to distill prompts or rollout preferences from stronger teacher models to improve student policies at scale~\citep{DBLP:conf/aaai/YangYYAYH025,DBLP:conf/nips/TsengWLI22,DBLP:journals/corr/abs-2506-02208}.
More recently, \citet{lu2025onpolicydistillation} distill a teacher's token-level logits on the student's rollouts to provide denser learning signals, and \citet{blackboxOPD} extend this setting to black-box teachers using objectives that rely only on teacher text while remaining on-policy~\citep{blackboxOPD}.
However, dense teacher supervision can be costly and brittle.
In contrast, we use the teacher only for annotating verifiable answers for newly generated queries, reducing reliance on dense teacher supervision and keeping augmentation calibrated to the student via on-policy generation.


\section{Preliminaries}
\label{sec:preliminaries}
We define the components of the RLVR pipeline and establish notation for how prompt sampling and pool maintenance interact with RL updates.
A \emph{query} (or \emph{prompt}) is an input sequence $x \in \mathcal{X}$, and the policy $\pi_{\theta}$ generates an output string $y \in \mathcal{Y}$.
Each query is associated with a ground-truth answer $g(x)$.
Following~\citet{DAPO}, we use a deterministic verifier to assign a scalar reward
\begin{equation}
r = R(x,y,g(x)) \in \{-1,0,1\},
\label{eq:rlvr-reward}
\end{equation}
where $R(\cdot)$ performs parsing and matching against $g(x)$.
We set $r=1$ if the parsed answer exactly matches $g(x)$, $r=0$ if it is incorrect, and $r=-1$ if the rollout is invalid.

\textbf{Prompt pool and sampling distribution.}
Let $\mathcal{D}_0$ denote a seed set of verifiable queries, each providing a prompt--answer pair $(x, g(x))$.
At training step $t$, RLVR maintains a bounded \emph{prompt pool} (or replay buffer) $\mathcal{P}_t=\{q_k\}_{k=1}^{|\mathcal{P}_t|}$ with capacity $|\mathcal{P}_t|\le N$.
In \name{}, $\mathcal{P}_t$ denotes an \emph{active} bounded pool used for sampling, while previously trained records may be temporarily moved to an archive for controlled recycling (\S\ref{sec:method:reinsert}).
Each \emph{query record} $q_k$ minimally stores $(x_k, g(x_k))$ and may additionally store training-time statistics.
We denote the verifiable per-rollout reward by $r_{i,j}\in\{-1,0,1\}$.

For pool maintenance, each record maintains a \emph{pool statistic} $\tilde r_k\in[-1,1]$ (e.g., the group mean shaped reward), which is used by the sampler to track empirical difficulty.
Newly inserted records (including on-policy augmentations) enter the pool as \emph{unscored} items with $\tilde r_k$ undefined.
After they are sampled once and their first rollout group is verified, we set $\tilde r_k$ and then treat them as scored items for subsequent sampling.
A sampling policy $\textsc{Sample}$ induces a categorical distribution over pool indices,
\begin{equation}
p_t(k) \triangleq \Pr(i=k \mid \mathcal{P}_t), \qquad k \in \{1,\dots,|\mathcal{P}_t|\},
\label{eq:pool_dist}
\end{equation}
which may depend on record statistics (e.g., $\{\tilde r_k\}$).
A mini-batch of prompts is drawn by sampling indices $i_1,\dots,i_B \sim p_t(\cdot)$ and forming
\begin{equation}
\mathcal{B}_t = \{x_{i_b}\}_{b=1}^{B}, \qquad |\mathcal{B}_t| = B,
\label{eq:batch}
\end{equation}
where $B$ is the batch size.
After generating rollouts and verifying rewards on $\mathcal{B}_t$, the corresponding records in $\mathcal{P}_t$ update their stored statistics (e.g., $\tilde r_k$), thereby coupling the next-step sampling distribution $p_{t+1}$ to policy progress.

\textbf{Rollouts and group-based RLVR.}
For each query $x_i \in \mathcal{B}_t$, we sample a \emph{group} of $n$ rollouts
\begin{equation}
y_{i,1},\dots,y_{i,n} \sim \pi_{\theta}(\cdot \mid x_i),
\label{eq:group-rollouts}
\end{equation}
and obtain rewards $r_{i,j} = R(x_i, y_{i,j}, g(x_i))$.
Group-based methods form an advantage estimate using a within-prompt baseline:
\begin{equation}
b_i = \frac{1}{n}\sum_{j=1}^{n} r_{i,j},
\qquad
A_{i,j} = r_{i,j} - b_i.
\label{eq:group-advantage}
\end{equation}
This captures learnability when rollouts within a group are mixed, producing non-degenerate advantages.

\textbf{Policy optimization objective.}
We optimize $\pi_\theta$ using a PPO/GRPO-style clipped surrogate objective.
Let $\pi_{\text{old}}$ denote the snapshot policy used to generate the current rollouts and define the importance ratio
\begin{equation}
\rho_{i,j}(\theta) = \frac{\pi_\theta(y_{i,j}\mid x_i)}{\pi_{\text{old}}(y_{i,j}\mid x_i)}.
\label{eq:ratio}
\end{equation}
We then optimize the clipped surrogate:
\begin{equation}
\mathcal{L}_{\text{clip}}(\theta) =
\mathbb{E}_{(i,j)}\!\left[
\min\Big(
\rho_{i,j}(\theta) A_{i,j},\;
\bar{\rho}_{i,j}(\theta) A_{i,j}
\Big)
\right],
\label{eq:ppo-clip}
\end{equation}
where $\bar{\rho}_{i,j}(\theta) \triangleq \operatorname{clip}\!\left(\rho_{i,j}(\theta), 1-\epsilon, 1+\epsilon\right)$ and $\epsilon>0$ is the clipping parameter.

\begin{figure*}[t]
    \centering
    \includegraphics[width=\linewidth]{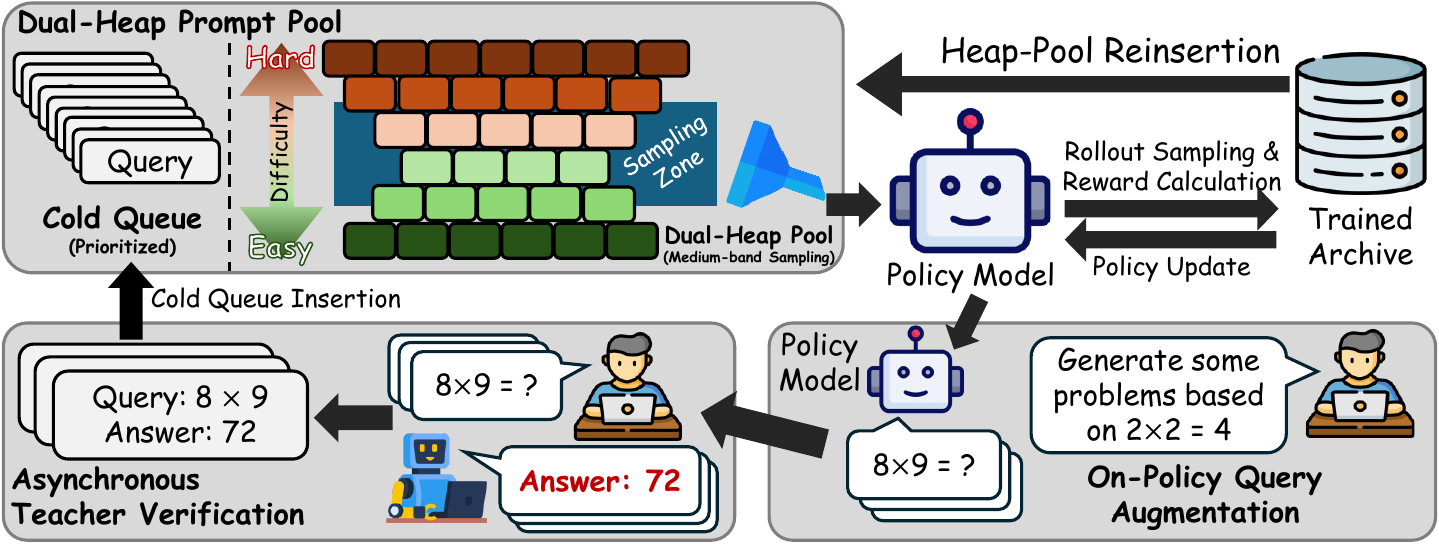}
    \caption{
    Overview of \name{}: a bounded prompt pool with a cold queue and difficulty-based heaps enables boundary sampling, while on-policy augmentation with asynchronous teacher verification grows the pool and controlled reinsertion recycles archived records.
    }
    \label{fig:heapa-overview}
\end{figure*}

\section{The \name{} Framework}
\label{sec:HeaPA_framework}
In this section, we propose \name{}, a query-lifecycle framework for RLVR that couples prompt sampling, on-policy pool growth, and stable integration.
As shown in Figure~\ref{fig:heapa-overview} and Algorithm~\ref{alg:heapa}, \name{} has five components:
(i) \textbf{dual-heap prompt pooling and boundary sampling} to draw $\mathcal{B}_t$ from $\mathcal{P}_t$ (\S\ref{sec:method:pool});
(ii) \textbf{RLVR rollout and policy update}, with pool-statistic updates and query archiving (\S\ref{sec:method:rlvr_update});
(iii) \textbf{on-policy query augmentation with asynchronous verification} to form new verifiable pairs $(x, g(x))$ (\S\ref{sec:method:aug});
(iv) \textbf{lineage-aware statistic re-estimation} over the augmentation graph $\mathcal{G}_t$ to stabilize correlated inserts (\S\ref{sec:method:refresh});
and (v) \textbf{controlled reinsertion} of refreshed archived records into $\mathcal{P}_t$ to smooth pool evolution (\S\ref{sec:method:reinsert}).

\vspace{-0.05in}
\subsection{Bounded Dual-Heap Prompt Pool}
\label{sec:method:pool}
\textbf{Pool state.}
We maintain a bounded prompt pool $\mathcal{P}_t$ with capacity $N$.
Following \S\ref{sec:preliminaries}, sampling draws indices $i \sim p_t(\cdot)$ over pool records.
\name{} instantiates $p_t$ using a dual-heap pool that tracks a boundary between lower- and higher-\emph{score} queries (harder/easier) using a stored per-record \emph{pool statistic} $\tilde r_k \in [-1,1]$.
This statistic is distinct from the per-rollout training reward $r_{i,j}$ as $\tilde r_k$ is an aggregate quantity maintained for sampling.

\textbf{Cold queue for unscored items.}
All seed queries from $\mathcal{D}_0$ enter $\mathcal{P}_t$ as \emph{cold} items with $\tilde r_k$ undefined.
Cold items are prioritized for sampling to quickly obtain verifier outcomes and initialize their pool statistics.
A record leaves the cold queue after its first verified rollout group, at which point $\tilde r_k$ becomes defined and the record is inserted into the scored partitions described below.

\textbf{Two partitions and the boundary.}
For scored items with defined $\tilde r_k$, we maintain two partitions: a low partition $\mathcal{P}_t^{\text{low}}$ and a high partition $\mathcal{P}_t^{\text{high}}$.
Intuitively, $\mathcal{P}_t^{\text{low}}$ contains harder items (lower $\tilde r_k$) and $\mathcal{P}_t^{\text{high}}$ contains easier items (higher $\tilde r_k$).
The \emph{boundary} is defined by the maximum score in $\mathcal{P}_t^{\text{low}}$ and the minimum score in $\mathcal{P}_t^{\text{high}}$.
We maintain a target fraction of scored items assigned to $\mathcal{P}_t^{\text{low}}$, denoted $\alpha\in(0,1)$, and rebalance partitions as scores change.
When one partition is empty (e.g., early in training or after capacity-constrained eviction), we fall back to sampling from the non-empty scored partition until both sides are populated, after which boundary sampling resumes.

\textbf{Boundary sampling.}
Given batch size $B$, \textsc{BoundarySample} draws (1) as many cold items as available (up to $B$), and (2) the remaining items from a \emph{boundary band} formed by candidates nearest the low/high split.
Concretely, we form a candidate set by alternating between the \emph{hardest easy} items (the minimum of $\mathcal{P}_t^{\text{high}}$) and the \emph{easiest hard} items (the maximum of $\mathcal{P}_t^{\text{low}}$) until collecting at most $2B$ candidates; we then select the required number uniformly from this set and reinsert non-chosen candidates back into their heaps.
This yields a frontier-focused distribution $p_t$ that concentrates probability mass on a medium-difficulty region without relying on offline difficulty scoring.

\textbf{Optional mixing with easy items.}
For stability, we optionally mix a small portion of very easy items (near the maximum of $\mathcal{P}_t^{\text{high}}$) into each batch, which helps preserve already learned behaviors while keeping most sampling mass on the boundary band.

\vspace{-0.05in}
\subsection{RLVR Update with Group-Based Advantages}
\label{sec:method:rlvr_update}
Given a sampled mini-batch $\mathcal{B}_t$ in Eq.~\eqref{eq:batch}, we generate groups of $n$ rollouts per query in Eq.~\eqref{eq:group-rollouts} and compute shaped training rewards using Eq.~\eqref{eq:rlvr-reward}.
We then compute within-prompt baselines and advantages in Eq.~\eqref{eq:group-advantage} and apply a standard group-based RLVR optimizer (GRPO or DAPO) using the clipped surrogate objective in Eq.~\eqref{eq:ppo-clip}, without modifying the underlying update rule.
In addition to per-rollout rewards, we store a per-record \emph{pool statistic} $\tilde r_k\in[-1,1]$ used by the sampler; for a query $x_i$ we set $\tilde r_i \leftarrow \frac{1}{n}\sum_{j=1}^{n} r_{i,j}$ (group mean shaped reward), and optionally maintain an exponential moving average over steps.
After updating $\tilde r_k$, we remove trained records from the active pool and move them into an archive for controlled recycling described in \S\ref{sec:method:reinsert}.
The archive stores records that are temporarily removed from the active pool after training; recycling reintroduces them in a controlled manner to prevent the pool from collapsing or becoming dominated by correlated inserts.

\vspace{-0.05in}
\subsection{On-Policy Query Augmentation with Asynchronous Verification}
\label{sec:method:aug}
A fixed seed pool can become stale as the policy improves.
To grow $\mathcal{P}_t$ in a frontier-aligned manner without shifting the bottleneck to expensive teacher supervision, \name{} performs on-policy query augmentation followed by lightweight asynchronous verification.

\textbf{On-policy augmentation.}
Conditioned on sampled prompts in $\mathcal{B}_t$, the current policy $\pi_\theta$ generates $n_{\text{aug}}$ candidate augmented queries.
Each candidate inherits a parent pointer to its source query record, inducing an augmentation lineage graph $\mathcal{G}_t$ over record IDs.

\textbf{Teacher verification.}
Candidates are sent to an asynchronous teacher model $\mathcal{T}$ that \emph{annotates} a verifiable final answer $g(x)$.
We discard candidates marked unsolvable, lacking a parsable numeric answer, or violating dataset filters.
The teacher is used to annotate answers for newly augmented queries.
This decouples training throughput from teacher latency: the RL loop proceeds while verification runs in parallel, and verified answers are drained and associated with their corresponding prompts for insertion into $\mathcal{P}_t$ whenever available.

\textbf{Insertion into the pool.}
Once $\mathcal{T}$ returns an answer for a candidate prompt $x$, we form a new verifiable pair $(x,g(x))$ using the policy-generated prompt $x$ and the teacher-annotated $g(x)$.
The resulting record enters the \emph{cold} queue with $\tilde r_k$ undefined and is prioritized for sampling.
After its first rollout group is verified, we set $\tilde r_k$ and insert the record into the scored partitions for subsequent boundary sampling.

\vspace{-0.05in}
\subsection{Lineage-Aware Pool Statistics Re-estimation}
\label{sec:method:refresh}
On-policy augmentations are often correlated through shared templates or local edits.
If inserted naively, correlated children can temporarily dominate sampling and induce oscillatory curricula.
\name{} mitigates this by maintaining a lineage graph $\mathcal{G}_t$ and periodically re-estimating the \emph{pool statistic} $\tilde r_k$ used for sampling (not the per-rollout training reward in Eq.~\eqref{eq:rlvr-reward}) via topology-aware propagation.

\textbf{Lineage graph and level construction.}
Let $\mathcal{G}_t$ be the directed parent-to-children graph induced by augmentation.
We construct level sets by propagating depths from leaves upward.
When the graph is acyclic, we compute depths in topological order; when cycles exist, we fall back to a bounded relaxation procedure and record a cycle-detected metric.
Nodes referenced by edges but missing from the current record index are treated as dangling; we track their count as a diagnostic of asynchronous insertion and eviction.

\textbf{Difficulty-aware aggregation.}
For an internal node $u$ with children $\mathcal{C}(u)$, we aggregate children scores bottom-up.
Let $\tilde r_c$ denote a child's most recent available score, preferring already-updated values within the same refresh pass.
We associate each child with a \emph{policy-estimated} relative difficulty factor $d_c$, produced at augmentation time by the student (e.g., the scalar in the \texttt{<DIFF>} field in Appendix~\ref{app:prompts:augment}), and clamp it to a bounded range $d_c \in [d_{\min}, d_{\max}]$.
For seed queries or records without a difficulty estimate, we set $d_c \leftarrow 1.0$.
We compute a difficulty-adjusted contribution
\begin{equation}
\phi(\tilde r_c, d_c) \triangleq \operatorname{clip}_{[-1,1]}\!\left(\frac{\tilde r_c}{\max(d_c,\epsilon)}\right),
\label{eq:difficulty_transform}
\end{equation}
so that children estimated to be harder (larger $d_c$) have their influence \emph{attenuated in magnitude} in the propagation.
We then instantiate the aggregator $\operatorname{Agg}(\cdot)$ either as an unweighted mean over children (\emph{children aggregation}) or as a leaf-count-weighted mean that upweights subtrees with more descendants (\emph{path aggregation}), optionally with a softmax normalization controlled by a temperature.
Finally, we blend the aggregated signal with the parent's previous score:
\begin{equation}
\tilde r_u \leftarrow \operatorname{clip}_{[-1,1]}\!\left( \tfrac{1}{2}\tilde r_u + \tfrac{1}{2}\operatorname{Agg}\big(\{\phi(\tilde r_c,d_c)\}_{c\in\mathcal{C}(u)}\big) \right).
\label{eq:refresh_blend}
\end{equation}
This refresh produces fractional $\tilde r_k \in [-1,1]$ even though rollout rewards remain discrete, improving boundary resolution for sampling.

\vspace{-0.05in}
\subsection{Controlled Reinsertion and Recycling}
\label{sec:method:reinsert}
To avoid overwhelming the pool with recently trained or newly verified items, \name{} uses controlled reinsertion.
After a record participates in training, it is placed into an archive.
We trigger recycling when the active pool would otherwise be unable to supply a full batch, or when the archive reaches a size threshold.
At recycling time, we refresh scores for eligible archived records using \S\ref{sec:method:refresh} and reinsert them into $\mathcal{P}_t$ in bounded batches, preventing transient shifts in $p_t$.
This provides two benefits: it amortizes refresh and insertion overhead to preserve throughput, and it smooths pool evolution by preventing sudden bursts of correlated items from dominating the boundary band.

\begin{table*}[t]
\centering
\resizebox{\textwidth}{!}{
\begin{tabular}{@{}p{6cm}|ccccccc|c@{}}
\toprule
Recipe / Variant & AIME24 & AIME25 & AMC23 & GPQA & MATH500 & Minerva & Olym.Bench & Avg. \\
\midrule
\multicolumn{9}{l}{\emph{Training Dataset: \textsc{DAPO-Math}}~\citep{DAPO}} \\
\midrule
GRPO (original;~\citealp{GRPO}) & 17.3 & 17.0 & 75.7 & 39.2 & 80.6 & 46.7 & 49.2 & 46.5 \\
GRPO + Prioritized Sampling (PS) & 16.7 & 16.5 & 76.8 & 38.8 & 81.6 & 47.8 & 51.2 & 47.1 \\
GRPO + Heap Sampling & 18.3 & 16.3 & 74.8 & 39.3 & 82.3 & 46.3 & 49.9 & 46.7 \\
GRPO + Policy Augmentation (PA) & 17.7 & 15.4 & 73.7 & 38.0 & 84.8 & 48.3 & 46.6 & 46.4 \\
GRPO + PS + PA & 19.9 & 18.4 & 76.3 & 40.0 & 85.3 & 49.5 & 50.3 & 48.5 \\
Reinforce-Ada~\citep{Reinforce-Ada} & \underline{22.0} & 19.8 & 81.0 & 41.9 & 84.9 & \underline{52.0} & 51.8 & 50.3 \\
\textbf{\name{} (GRPO)} + ChildAgg & 16.4 & \textbf{\underline{22.4}} & 77.8 & 41.8 & 85.2 & 45.0 & 52.2 & 48.7 \\
\textbf{\name{} (GRPO)} + PathAgg & 21.4 & 21.0 & \underline{82.4} & \textbf{\underline{42.7}} & \textbf{\underline{85.6}} & 51.4 & \textbf{\underline{52.9}} & \textbf{\underline{51.1}} \\
\midrule
DAPO (original;~\citealp{DAPO}) & 19.6 & 18.0 & 76.5 & 38.1 & 76.7 & 47.9 & 45.9 & 46.1 \\
DAPO + Prioritized Sampling (PS) & 19.4 & 18.3 & 74.1 & 37.0 & 76.5 & 48.8 & 44.2 & 45.5 \\
DAPO + Heap Sampling & 19.0 & 16.5 & 75.3 & 36.9 & 75.1 & 46.3 & 44.2 & 44.8 \\
DAPO + Policy Augmentation (PA) & 19.0 & 16.4 & 76.5 & 38.9 & 75.9 & 47.5 & 45.3 & 45.6 \\
DAPO + PS + PA & 19.6 & 18.7 & 78.4 & 39.4 & 75.3 & 48.6 & 46.1 & 46.6 \\
\textbf{\name{} (DAPO)} + ChildAgg & 19.4 & 20.4 & \textbf{\underline{83.4}} & 40.9 & 76.9 & 48.6 & 46.6 & 48.0 \\
\textbf{\name{} (DAPO)} + PathAgg & \underline{23.3} & \underline{21.7} & 80.8 & \underline{42.0} & \underline{78.0} & \underline{50.6} & \underline{48.8} & \underline{49.3} \\
\midrule
\multicolumn{9}{l}{\emph{Training Dataset: \textsc{OpenR1-Math}}~\citep{OpenR1}} \\
\midrule
GRPO (original;~\citealp{GRPO}) & 16.5 & 11.5 & 61.7 & 37.0 & 76.4 & 36.0 & 40.2 & 39.9 \\
GRPO + Prioritized Sampling (PS) & 18.0 & 13.0 & 63.8 & 37.6 & 77.4 & 37.1 & 44.2 & 41.6 \\
GRPO + Heap Sampling & 18.1 & 12.0 & 64.3 & 36.1 & 76.7 & 35.8 & 45.5 & 41.2 \\
GRPO + Policy Augmentation (PA) & 16.9 & 11.9 & 63.7 & 36.8 & 77.6 & 35.6 & 44.6 & 41.0 \\
GRPO + PS + PA & 17.1 & 14.9 & 66.3 & 36.8 & 78.1 & 35.8 & 43.3 & 41.8 \\
Reinforce-Ada~\citep{Reinforce-Ada} & 19.8 & \underline{16.4} & 67.5 & 38.0 & \underline{81.6} & 37.2 & 46.0 & 43.8 \\
\textbf{\name{} (GRPO)} + ChildAgg & 17.7 & 15.6 & 63.7 & 36.9 & 79.2 & 36.5 & \underline{46.2} & 42.3 \\
\textbf{\name{} (GRPO)} + PathAgg & \underline{20.6} & 15.5 & \underline{68.4} & \underline{38.5} & 81.3 & \underline{37.7} & 45.9 & \underline{44.0} \\
\midrule
DAPO (original;~\citealp{DAPO}) & 19.5 & 15.4 & 65.4 & 37.8 & 80.6 & 51.1 & 45.1 & 45.0 \\
DAPO + Prioritized Sampling (PS) & 20.3 & 16.7 & 64.0 & 37.7 & 80.4 & 53.0 & 43.4 & 45.1 \\
DAPO + Heap Sampling & 16.0 & 15.3 & 64.2 & 36.9 & 83.0 & 51.8 & 44.6 & 44.5 \\
DAPO + Policy Augmentation (PA) & 19.9 & 14.8 & 66.4 & 36.6 & \underline{84.8} & 53.7 & 44.5 & 45.8 \\
DAPO + PS + PA & 20.5 & 16.1 & 68.3 & 38.1 & 82.2 & \textbf{\underline{54.8}} & 45.3 & 46.5 \\
\textbf{\name{} (DAPO)} + ChildAgg & 25.6 & \underline{17.1} & 68.3 & 37.0 & 80.4 & 52.7 & 47.5 & 46.9 \\
\textbf{\name{} (DAPO)} + PathAgg & \textbf{\underline{25.7}} & \underline{17.1} & \underline{69.4} & \underline{39.8} & 84.2 & \textbf{\underline{54.8}} & \underline{48.5} & \underline{48.5} \\
\bottomrule
\end{tabular}
}
\caption{Main evaluation results on seven math benchmarks of Qwen2.5-7B-Instruct trained using~\name{} and several baselines. 
\underline{Underlined} numbers mark block-wise best results; \textbf{bold} numbers mark column-wise best results.}
\label{table:main-evaluation-results}
\end{table*}

\vspace{-0.05in}
\section{Main Evaluations of~\name{}}
\label{sec:experiments}
\subsection{Experiment Setup}
\label{sec:experiment_setup}
We use Qwen2.5-7B-Instruct~\citep{Qwen2.5} as the backbone model in our main experiments to validate the effectiveness of~\name{}, and study several ablations of~\name{}.
Additionally, we include Reinforce-Ada~\citep{Reinforce-Ada} and prioritized sampling (PS;~\citealp{Kimi1.5}) as two sampling-based baselines.
We use two math corpora as sources of seed queries: DAPO-Math~\citep{DAPO} and OpenR1-Math-220k (the default subset;~\citealp{OpenR1}). 
Please see more details in Appendix~\ref{appendix:implementation_details}.

\textbf{Evaluation benchmarks.}
We evaluate our models on a set of held-out benchmarks spanning competition math and broader math reasoning:
AIME24~\citep{huggingfaceh4_aime2024}, AIME25~\citep{opencompass_aime2025}, AMC23~\citep{mathai_amc23}, GPQA~\citep{GPQA}, MATH500~\citep{DBLP:conf/iclr/LightmanKBEBLLS24}, MinervaMath~\citep{MinervaMath}, and OlympiadBench~\citep{OlympiadBench}.
Dataset statistics are shown in Table~\ref{tab:benchmark-sizes}.
We report Avg@16, defined as the mean correctness over 16 rollouts, as our evaluation metric.

\subsection{Evaluation Results}
\label{sec:main_results}
Results are reported in Table~\ref{table:main-evaluation-results}, and our findings are:

\textbf{Overall gains from the full \name{} design.}
Across both training corpora and both optimizers,~\name{} achieves the strongest \emph{average} performance.
On \textsc{DAPO-Math}, it reaches \textbf{51.1} under GRPO (vs.\ 46.5; +4.6) and \underline{49.3} under DAPO (vs.\ 46.1; +3.2); on \textsc{OpenR1-Math}, it improves GRPO to \underline{44.0} (vs.\ 39.9; +4.1) and DAPO to \underline{48.5} (vs.\ 45.0; +3.5).
Reinforce-Ada is a strong GRPO baseline (50.3 on \textsc{DAPO-Math} and 43.8 on \textsc{OpenR1-Math}), but remains below \name{} on average, despite occasionally winning on one or two benchmarks.
We report additional three-seed robustness and teacher-quality analyses in Appendices~\ref{app:multi_seed} and~\ref{app:teacher_quality}: aggregate gains exceed seed variability, accepted augmented queries have a \textbf{1.2\%} estimated wrong-label rate after filtering, and synthetic label-noise degradation is gradual.

\textbf{Ablations: sampling or augmentation alone is not sufficient.}
Adding only prioritized sampling (PS) or only policy augmentation (PA) yields smaller and less consistent gains than the full \name{} stack.
For example, on \textsc{DAPO-Math} with GRPO, PS+PA reaches 48.5, still below 51.1 with \name{}, and on \textsc{OpenR1-Math} with DAPO, PS+PA reaches 46.5 compared with 48.5 for \name{}.
These results suggest that the gains arise from the \emph{coupled} design---frontier-focused sampling together with on-policy pool growth and stable integration---rather than from any single component in isolation.

\textbf{Comparison with direct difficulty-based curriculum.}
We further compare against a PCL-style curriculum baseline that uses the updated policy as an online difficulty estimator to select intermediate-difficulty prompts from a fixed pool.
This baseline is strong and directly aligned with the frontier-tracking motivation, reaching 50.3 Avg@16 on \textsc{DAPO-Math} and 43.5 on \textsc{OpenR1-Math} under GRPO, but \name{} remains competitive or better with 51.1 and 44.0 respectively.
This suggests that intermediate-difficulty selection is important, but the evolving-pool mechanism and lineage-aware reinsertion provide additional benefits beyond fixed-pool curriculum selection.

\textbf{Topology-aware pool statistics re-estimation strengthens frontier tracking.}
Under the same sampling strategy, \emph{PathAgg} consistently matches or exceeds \emph{ChildAgg} on average.
For instance, on \textsc{DAPO-Math} with GRPO, PathAgg improves over ChildAgg by +2.4 average points (51.1 vs.\ 48.7), and on \textsc{OpenR1-Math} with GRPO by +1.7 (44.0 vs.\ 42.3).
This supports that correlated augmentations can distort sampling priorities, and that topology-aware propagation yields a smoother and more reliable pool statistic $\tilde r$ for boundary sampling.

\begin{table*}[t]
\centering
\resizebox{\textwidth}{!}{%
\begin{tabular}{@{}lp{2.8cm}|ccccccc|c@{}}
\toprule
Backbone & Method & AIME24 & AIME25 & AMC23 & GPQA & MATH500 & MinervaMath & Olym.Bench & Avg. \\
\midrule
Qwen3-0.6B & GRPO & 5.9 & 3.3 & 28.7 & 31.1 & 40.9 & 6.6 & 14.2 & 18.7 \\
Qwen3-0.6B & GRPO (w. PS) & 6.7 & 3.9 & 30.6 & \underline{31.8} & 43.1 & 8.0 & 14.6 & 19.8 \\
Qwen3-0.6B & \textbf{\name{}} & \underline{7.3} & \underline{4.3} & \underline{31.9} & 31.6 & \underline{44.6} & \underline{8.9} & \underline{14.9} & \underline{20.5} \\
\midrule
Qwen3-1.7B & GRPO & 10.0 & \underline{6.2} & 39.8 & 32.3 & 58.0 & 14.8 & 26.5 & 26.8 \\
Qwen3-1.7B & GRPO (w. PS) & 11.0 & \underline{6.2} & \underline{40.6} & 32.9 & 62.0 & 15.6 & 28.2 & 28.1 \\
Qwen3-1.7B & \textbf{\name{}} & \underline{11.6} & \underline{6.2} & 40.3 & \underline{33.3} & \underline{64.7} & \underline{16.1} & \underline{29.4} & \underline{28.8} \\
\midrule
Qwen3-4B & GRPO & 20.4 & 18.1 & \underline{58.9} & 41.8 & 59.3 & 21.0 & 33.1 & 36.1 \\
Qwen3-4B & GRPO (w. PS) & 21.9 & \underline{19.6} & 58.7 & 43.3 & 70.9 & 28.0 & 41.1 & 40.5 \\
Qwen3-4B & \textbf{\name{}} & \underline{22.9} & 19.3 & 58.6 & \underline{44.3} & \underline{78.7} & \underline{32.6} & \underline{46.5} & \underline{43.3} \\
\midrule
Qwen3-8B & GRPO & 25.0 & 21.7 & 68.0 & 45.6 & 70.8 & 28.9 & 39.4 & 42.8 \\
Qwen3-8B & GRPO (w. PS) & 25.3 & 21.9 & 68.6 & \textbf{\underline{50.0}} & 76.8 & 32.4 & 45.6 & 45.8 \\
Qwen3-8B & \textbf{\name{}} & \textbf{\underline{25.6}} & \textbf{\underline{22.0}} & \textbf{\underline{69.1}} & 49.3 & \textbf{\underline{81.8}} & \textbf{\underline{35.3}} & \textbf{\underline{50.7}} & \textbf{\underline{47.7}} \\
\bottomrule
\end{tabular}
}
\caption{Parameter scaling analysis on Qwen3 backbones with GRPO, GRPO+prioritized sampling (PS), and~\name{} on seven benchmarks.}
\vspace{-0.1in}
\label{table:parameter-scaling}
\end{table*}

\section{Analysis}
\subsection{Scaling Transfer and Query-Source Analysis}
\label{sec:analysis:scaling_and_dynamics}
A key design goal of \name{} is to be \emph{optimizer-agnostic} and \emph{backbone-agnostic} in implementation: the same query-lifecycle mechanisms can be plugged into different RLVR recipes and model backbones.
To test this, we transfer \name{} to a different model family (Qwen3;~\citealp{Qwen3}) at four parameter scales, and further analyze how verified on-policy augmentations differ from seed queries in terms of reward and advantage signals over training.

\textbf{Setup.}
For each Qwen3 backbone, we train with the same GRPO recipe under three variants: (i) the original GRPO sampling baseline, (ii) GRPO with prioritized sampling (PS), and (iii) GRPO with \name{} replacing the query sourcing and sampling subsystem.
We evaluate all trained models on seven benchmarks and report Avg@16.
In addition, we log query-source statistics during training by separating batches into \emph{seed queries} and \emph{teacher-verified augmented queries}, and track their (i) mean reward and (ii) advantage-related signal over steps to understand how on-policy pool growth changes the learning signal distribution.

\begin{figure}[t]
    \centering
    \includegraphics[width=\linewidth]{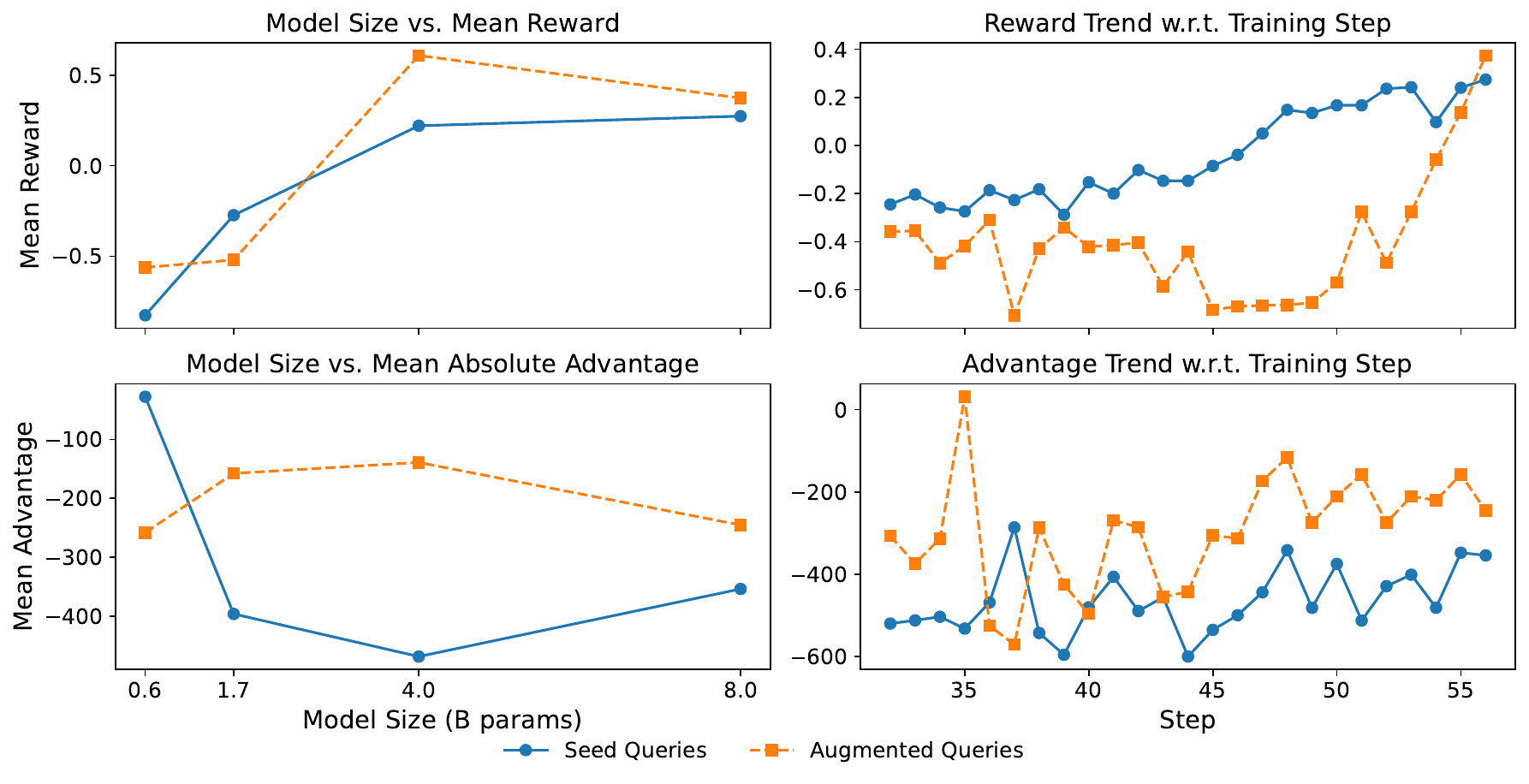}
    \caption{Comparison of seed and augmented query statistics across model scales and training steps.}
    \vspace{-0.1in}
    \label{fig:query-reward-advantage-2x2}
\end{figure}

\textbf{Results.}
Table~\ref{table:parameter-scaling} shows that \name{} consistently improves over both GRPO and GRPO+PS at all four scales, indicating that the method transfers cleanly to a different backbone family.
On Avg@16, \name{} improves over GRPO by +1.8 (0.6B), +2.0 (1.7B), +7.2 (4B), and +4.9 (8B), and it also outperforms GRPO+PS at every scale.
The largest gains appear at mid-to-large scales (4B/8B), where \name{} yields broad improvements across benchmarks, consistent with the intuition that better frontier tracking and stable on-policy pool growth become increasingly valuable as the model becomes capable of solving a larger fraction of the pool.
Complementing the aggregate benchmark results, Figure~\ref{fig:query-reward-advantage-2x2} reveals a clear query-source pattern: augmented queries tend to start with lower reward than seed queries (harder/less familiar), but their reward rises sharply over training and with model scale, indicating that verified on-policy augmentation progressively supplies frontier-aligned queries that become learnable as the policy improves; the advantage-related signal exhibits a corresponding shift, suggesting that augmented queries increasingly contribute informative (non-saturated) rollout groups rather than collapsing to all-correct or all-incorrect outcomes.

\begin{figure}[t]
    \centering
    \includegraphics[width=\linewidth]{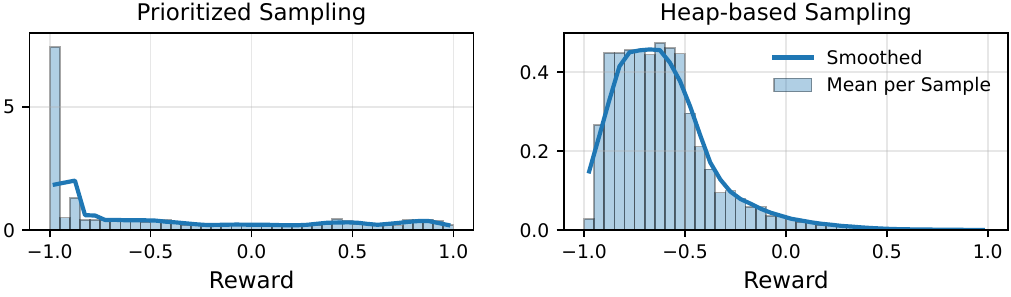}
    \caption{Comparison of probability mass over reward for two sampling strategies. Rewards are computed by averaging rollouts with a frozen Qwen2.5-7B-Instruct model.}
    \label{fig:sampling-mass-threepanels}
\end{figure}
\subsection{Query Sampling Probability Distribution}
\label{sec:analysis:sampling_mass}
An important design goal of \name{} is to focus training on the evolving capability frontier, i.e., medium-difficulty queries that are neither always correct nor always incorrect. 
This is especially important for group-based RLVR, since medium-difficulty queries tend to produce mixed-success rollout groups and thus more informative advantage signals than saturated all-correct or all-wrong groups. 
To diagnose whether a sampler exhibits this behavior, we freeze Qwen2.5-7B-Instruct to define a fixed reward landscape, repeatedly sample queries from the same pool snapshot under each strategy, and compute rewards with identical verification/decoding settings. 
Figure~\ref{fig:sampling-mass-threepanels} shows that prioritized sampling concentrates probability mass on the extreme low-reward tail (over-selecting very hard queries), whereas our heap-based boundary sampling shifts mass toward a medium-reward band, consistent with the intended boundary-focused curriculum.

\begin{wrapfigure}{l}{0.5\linewidth}
    \centering
    \vspace{-0.2in}
    \includegraphics[width=\linewidth]{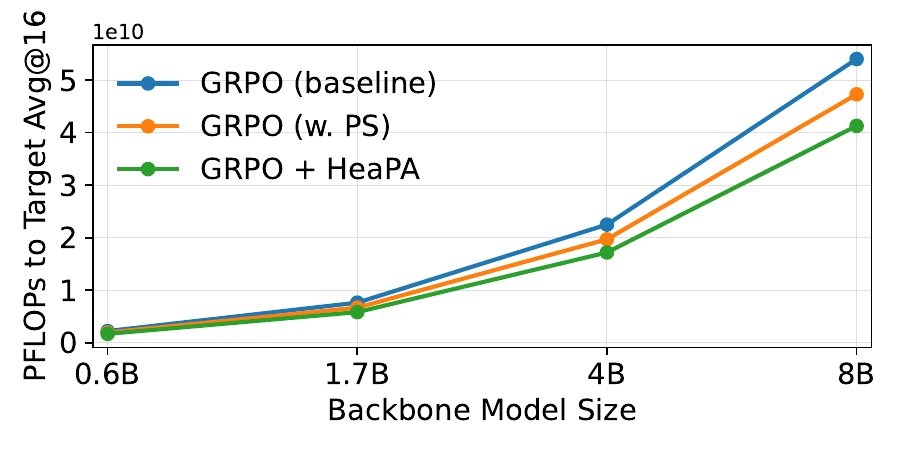}
    \vspace{-0.3in}
    \caption{Training compute (PFLOPs) required to \emph{first reach} the baseline GRPO target performance (final Avg@16 of the baseline run) for Qwen3 backbones at four scales. Lower is better.}
    \vspace{-0.2in}
    \label{fig:flops_to_target_qwen3}
\end{wrapfigure}
\subsection{Training Cost and Efficiency}
We further analyze the cost and efficiency of \name{} on Qwen3 backbones at four scales (0.6B/1.7B/4B/8B), comparing the baseline GRPO pipeline, GRPO with prioritized sampling (PS), and GRPO with \name{}.
Since RLVR is typically dominated by rollout generation, we focus on (i) the \emph{training compute} needed to reach a fixed target performance and (ii) the \emph{runtime overhead} introduced by pool maintenance and augmentation handling.

\textbf{PFLOPs to reach the baseline target.}
For each scale, we set the target as the final Avg@16 achieved by the baseline run, and measure the total training PFLOPs required by each method to \emph{first} reach this target under the same evaluation protocol.
Figure~\ref{fig:flops_to_target_qwen3} plots PFLOPs-to-target across scales: \name{} consistently reduces compute to reach the baseline target, suggesting that boundary sampling and on-policy pool growth improve sample-efficiency.

\textbf{Wall-clock and teacher-side overhead.}
Averaged across all scales, \name{} increases per-step wall-clock time by only \textbf{2.1\%}, since heap operations are lightweight and teacher verification runs asynchronously outside the rollout/update critical path.
The teacher is used only as an answer-level annotator for newly generated candidate queries, rather than for dense rollout supervision or token-level distillation.
When counted as separate auxiliary service time, asynchronous teacher annotation takes about \textbf{13.1\%} of the main training time; teacher-quality and label-noise analyses are reported in Appendix~\ref{app:teacher_quality}.

\section{Conclusions}
In this paper, we introduce \name{}, a query-lifecycle framework for RLVR that maintains a bounded, evolving prompt pool, tracks the moving capability frontier via heap-based boundary sampling, and grows the pool on-policy through augmentation with lightweight asynchronous verification. 
Empirical results show consistent accuracy gains and improved training efficiency, and ablations confirm that the strongest improvements come from coupling frontier-focused sampling with on-policy pool growth while stabilizing correlated inserts via topology-aware pool statistics re-estimation. 
We hope \name{} becomes a plug-and-play foundation for scalable RLVR post-training, improving the efficiency of rollout allocation and accelerating reliable reasoning progress as training scales.


\clearpage
\section*{Ethics Statement}
This paper presents work whose goal is to advance the efficiency and stability of reinforcement learning with verifiable rewards (RLVR) for large language models.
Our contribution focuses on training-time prompt sampling and on-policy data augmentation with verifiable answers, and does not introduce new model capabilities beyond those enabled by stronger RLVR pipelines.
We do not anticipate distinct ethical concerns specific to the proposed sampling and data-management framework beyond the broader considerations that apply to developing and deploying language models.
We encourage practitioners to follow established best practices for responsible model development and deployment, including careful evaluation, safety mitigations, and appropriate usage policies.




\bibliography{colm2026_conference}
\bibliographystyle{colm2026_conference}

\clearpage
\appendix

\begin{algorithm}[t]
\caption{The \name{} Training Algorithm}
\label{alg:heapa}
\footnotesize
\begin{algorithmic}[1]
\Require Seed data $\mathcal{D}_0$, initial policy $\pi_\theta$, verifier $R$, teacher $\mathcal{T}$,
pool capacity $N$, batch size $B$, group size $n$, augmentations $n_{\text{aug}}$, steps $T$
\State Initialize bounded dual-heap pool $\mathcal{P}_0$ with capacity $N$ (\S\ref{sec:method:pool})
\State Insert seed records $(x,g(x))\in\mathcal{D}_0$ into the \textbf{cold} queue of $\mathcal{P}_0$ (unscored), subsampling if $|\mathcal{D}_0|>N$
\State Start asynchronous verification workers connected to teacher $\mathcal{T}$ (\S\ref{sec:method:aug})
\State Initialize archive $\mathcal{A}\leftarrow \emptyset$ and lineage graph $\mathcal{G}\leftarrow \emptyset$
\For{$t = 1$ \textbf{to} $T$}
    \State Drain teacher-annotated results for pending candidate prompts; for each verified candidate prompt $x^{\text{aug}}$,
    form $(x^{\text{aug}}, g(x^{\text{aug}}))$ and insert it into the \textbf{cold} queue of $\mathcal{P}_{t-1}$ (with $\tilde r_k$ undefined)
    \If{\textsc{ShouldRecycle}$(t,\mathcal{P}_{t-1},\mathcal{A})$}
        \State $\{\tilde r_k^{\text{ref}}\}_{q_k\in\mathcal{A}} \gets \textsc{RefreshPoolStatistic}(\mathcal{G}, \mathcal{A})$
        \Comment{\S\ref{sec:method:refresh}}
        \State \textsc{ReinsertBatched}$(\mathcal{P}_{t-1}, \mathcal{A}, \{\tilde r_k^{\text{ref}}\})$
        \Comment{\S\ref{sec:method:reinsert}}
    \EndIf
    \State $\mathcal{B}_t \gets \textsc{BoundarySample}(\mathcal{P}_{t-1}, B)$
    \Comment{Induces $p_t$ in Eq.~\eqref{eq:pool_dist}}
    \State Sample $n$ rollouts and rewards with snapshot $\pi_{\text{old}}\!\leftarrow\!\pi_\theta$:
    $y_{i,1:n}\!\sim\!\pi_{\text{old}}(\cdot\mid x_i)$, $r_{i,j}\!\gets\!R(x_i,y_{i,j},g(x_i))$
    \State Update per-record $\tilde r$; move trained records $\mathcal{P}\!\rightarrow\!\mathcal{A}$; update $\theta$ with group-based RLVR
    \If{\textsc{AugmentOn}$(t)$ \textbf{and} $|\mathcal{P}_{t-1}| < N$}
        \State Propose augmented queries $\tilde{\mathcal{B}} \gets \textsc{Augment}(\pi_\theta, \mathcal{B}_t, n_{\text{aug}})$
        \State Add parent$\rightarrow$child edges to $\mathcal{G}$; enqueue $\tilde{\mathcal{B}}$ for asynchronous teacher verification
    \EndIf
    \State We update the pool in-place; denote the resulting pool after step $t$ as $\mathcal{P}_t$
\EndFor
\State \Return trained policy $\pi_\theta$
\end{algorithmic}
\end{algorithm}

\section{Implementation Details}
\label{appendix:implementation_details}
For training datasets, we filter out non-English queries, queries with non-numerical answers, and queries longer than 2,048 tokens.
We use the problem statement as the prompt $x$ and the final answer as the ground truth $g(x)$.
After filtering, DAPO-Math contains $\approx 14$K prompts and OpenR1-Math-220k contains $\approx 93$K prompts for training.

We build all training code on \textsc{verl}\footnote{\href{https://github.com/verl-project/verl}{https://github.com/verl-project/verl}}.
For our main experiments, we use Qwen/Qwen2.5-7B-Instruct as the backbone model.
We train for 500 steps with DAPO and 1{,}000 steps with GRPO, using a batch size of 512 prompts and generating 16 rollouts per prompt.
The learning rate is set to $10^{-6}$ with cosine decay and no warm-up.
We set the maximum prompt length to 2{,}048 and the maximum response length to 20{,}480 (12K for Qwen3).
KL loss is applied only during GRPO training, with a coefficient of $10^{-3}$.
All other parameters not mentioned above are kept the same as in the original GRPO/DAPO training scripts~\citep{GRPO,DAPO}.

For \name{}-specific hyperparameters, we set the query pool capacity to 100,000 and generate two augmentations per source query.
We use \textit{GPT-5-nano} as the teacher model for asynchronous verification, chosen for its balance of capability and cost efficiency; manual inspection indicates it provides sufficiently accurate ground-truth answer annotations.
For query augmentation, we set the generation token limit to 512.
All of our models are trained using a total of 96 NVIDIA H200 GPUs.

\section{Multi-Seed Robustness}
\label{app:multi_seed}
Small benchmarks such as AIME24, AIME25, and AMC23 can exhibit non-trivial seed-induced variability.
To assess whether the main trends are robust, we rerun the key GRPO configurations using three training seeds: 0, 999, and 2026.
The training corpus, backbone, rollout budget, evaluation prompts, decoding protocol, and number of training steps are kept fixed.
Table~\ref{tab:multi_seed_robustness} reports mean and standard deviation across training seeds.

\begin{table*}[t]
\centering
\resizebox{\textwidth}{!}{
\begin{tabular}{@{}lllrcccc@{}}
\toprule
Training Dataset & Optimizer & Method & \# Seeds & AIME24 & AIME25 & AMC23 & Avg. \\
\midrule
\textsc{DAPO-Math} & GRPO & GRPO & 3 & $16.8 \pm 0.7$ & $17.6 \pm 0.5$ & $74.9 \pm 0.9$ & $46.1 \pm 0.4$ \\
\textsc{DAPO-Math} & GRPO & Reinforce-Ada & 3 & $22.7 \pm 0.8$ & $19.2 \pm 0.6$ & $81.8 \pm 0.7$ & $49.8 \pm 0.5$ \\
\textsc{DAPO-Math} & GRPO & \name{} + PathAgg & 3 & $21.9 \pm 0.6$ & $21.6 \pm 0.7$ & $83.1 \pm 0.8$ & $\mathbf{51.4 \pm 0.4}$ \\
\midrule
\textsc{OpenR1-Math} & GRPO & GRPO & 3 & $15.9 \pm 0.9$ & $12.1 \pm 0.6$ & $60.8 \pm 0.8$ & $39.5 \pm 0.5$ \\
\textsc{OpenR1-Math} & GRPO & Reinforce-Ada & 3 & $20.4 \pm 0.7$ & $16.9 \pm 0.8$ & $66.9 \pm 0.6$ & $44.2 \pm 0.4$ \\
\textsc{OpenR1-Math} & GRPO & \name{} + PathAgg & 3 & $21.1 \pm 0.6$ & $15.9 \pm 0.7$ & $69.2 \pm 0.8$ & $\mathbf{44.5 \pm 0.5}$ \\
\bottomrule
\end{tabular}
}
\caption{
Multi-seed robustness results for key GRPO configurations using Qwen2.5-7B-Instruct.
We report mean $\pm$ standard deviation across three training seeds.
}
\label{tab:multi_seed_robustness}
\end{table*}

The multi-seed results support the main conclusion.
On \textsc{DAPO-Math}, \name{} improves the 7-benchmark average over GRPO by +5.3 points and over Reinforce-Ada by +1.6 points, which are larger than the observed standard deviations.
On \textsc{OpenR1-Math}, \name{} improves over GRPO by +5.0 points, while the gap to Reinforce-Ada is smaller (+0.3 points).
Thus, we interpret the OpenR1-Math comparison with Reinforce-Ada as competitive rather than a large win.
Overall, the aggregate trend remains stable: \name{} consistently improves over the original GRPO baseline and remains competitive with or stronger than strong adaptive sampling baselines.

\section{Additional Related Baselines: PCL-style Curriculum and SvS-style Augmentation}
\label{app:related_baselines}

We add two directly related baselines to further clarify the design space around \name{}:
(i) a PCL-style curriculum baseline that selects intermediate-difficulty prompts using a policy-dependent difficulty signal, and
(ii) an SvS-style teacher-free answer-preserving augmentation baseline.
Both are evaluated under the same GRPO setup as the main experiments.

\subsection{PCL-style Curriculum Baseline}

Prompt curriculum learning (PCL) selects intermediate-difficulty prompts for the current policy.
In our implementation, the updated policy serves as the online value model or difficulty estimator for prompt selection.
The prompt pool is kept fixed, so this baseline isolates difficulty-based curriculum selection from \name{}'s evolving-pool mechanism.
Table~\ref{tab:pcl_style_baseline} compares the PCL-style baseline with GRPO, Reinforce-Ada, and \name{}.

\begin{table*}[t]
\centering
\resizebox{\textwidth}{!}{
\begin{tabular}{@{}lllccc@{}}
\toprule
Training Dataset & Optimizer & Method & Difficulty Signal & Pool Growth? & Avg@16 \\
\midrule
\textsc{DAPO-Math} & GRPO & GRPO & Uniform sampling & No & 46.5 \\
\textsc{DAPO-Math} & GRPO & Reinforce-Ada & Adaptive rollout uncertainty & No & 50.3 \\
\textsc{DAPO-Math} & GRPO & PCL-style Curriculum & Policy-dependent prompt difficulty & No & 50.3 \\
\textsc{DAPO-Math} & GRPO & \name{} + PathAgg & Empirical heap-boundary statistic & Yes & \textbf{51.1} \\
\midrule
\textsc{OpenR1-Math} & GRPO & GRPO & Uniform sampling & No & 39.9 \\
\textsc{OpenR1-Math} & GRPO & Reinforce-Ada & Adaptive rollout uncertainty & No & 43.8 \\
\textsc{OpenR1-Math} & GRPO & PCL-style Curriculum & Policy-dependent prompt difficulty & No & 43.5 \\
\textsc{OpenR1-Math} & GRPO & \name{} + PathAgg & Empirical heap-boundary statistic & Yes & \textbf{44.0} \\
\bottomrule
\end{tabular}
}
\caption{
Comparison with a PCL-style fixed-pool curriculum baseline.
The PCL-style baseline selects intermediate-difficulty prompts using an online policy-dependent difficulty estimate, while \name{} additionally supports on-policy pool growth and lineage-aware reinsertion.
}
\label{tab:pcl_style_baseline}
\end{table*}

The PCL-style curriculum is a strong and directly relevant baseline.
On \textsc{DAPO-Math}, it reaches 50.3 Avg@16, matching Reinforce-Ada and substantially improving over GRPO.
On \textsc{OpenR1-Math}, it reaches 43.5 Avg@16, close to Reinforce-Ada's 43.8.
\name{} remains competitive or better in both settings, reaching 51.1 and 44.0 respectively.
This suggests that intermediate-difficulty selection is important, but selection alone does not fully explain the gains; evolving-pool growth and stable reintegration of correlated augmentations provide additional benefits.

\subsection{SvS-style Teacher-Free Answer-Preserving Augmentation}

We also compare against a teacher-free answer-preserving augmentation baseline inspired by self-play with variational synthesis.
For this baseline, augmented queries are generated only from prompts for which the current policy has produced at least one correct solution.
The augmentation instruction is constrained to preserve the original final answer, so no teacher annotation is used.
The resulting answer-preserving queries are inserted into the same cold queue and empirically scored after their first verified rollout group.
Table~\ref{tab:svs_style_baseline} compares teacher-free augmentation alone and inside the \name{} pool-management framework.

\begin{table*}[t]
\centering
\resizebox{\textwidth}{!}{
\begin{tabular}{@{}lccccc@{}}
\toprule
Method & Teacher Needed? & Answer Type & Heap Boundary Sampling? & PathAgg / Reinsertion? & DAPO-Math / OpenR1-Math Avg@16 \\
\midrule
GRPO & No & Original seed answer & No & No & 46.5 / 39.9 \\
GRPO + PA & Yes & Teacher-annotated new answer & No & No & 46.4 / 41.0 \\
GRPO + SvS-style PA & No & Preserved original answer & No & No & 47.8 / 42.4 \\
\name{} w/ SvS-style PA & No & Preserved original answer & Yes & Yes & 51.0 / 42.9 \\
\name{} & Yes & Teacher-annotated new answer & Yes & Yes & \textbf{51.1 / 44.0} \\
\bottomrule
\end{tabular}
}
\caption{
Comparison with SvS-style teacher-free answer-preserving augmentation under GRPO.
The two numbers in the last column denote Avg@16 on \textsc{DAPO-Math} and \textsc{OpenR1-Math}, respectively.
}
\label{tab:svs_style_baseline}
\end{table*}

The results provide two clarifications.
First, teacher-free answer-preserving augmentation is useful, improving GRPO from 46.5 to 47.8 on \textsc{DAPO-Math} and from 39.9 to 42.4 on \textsc{OpenR1-Math}.
Second, \name{}'s pool-management mechanism remains beneficial even when augmentation is teacher-free: adding heap boundary sampling, PathAgg, and controlled reinsertion raises the SvS-style variant to 51.0 on \textsc{DAPO-Math} and 42.9 on \textsc{OpenR1-Math}.
The comparison between \name{} w/ SvS-style PA and full \name{} further clarifies the role of teacher annotation.
On \textsc{DAPO-Math}, the two variants are nearly tied, while on \textsc{OpenR1-Math}, full \name{} improves from 42.9 to 44.0.
This suggests that the main contribution is the \name{} query lifecycle, while teacher-annotated new-answer augmentation provides additional flexibility in settings where answer-preserving transformations are less sufficient.

\section{Teacher Quality and Label-Noise Robustness}
\label{app:teacher_quality}

\name{} uses the teacher only as an answer-level annotator for newly generated candidate queries.
The teacher does not provide dense rollout rewards, token-level logits, process supervision, or preference labels.
Before insertion into the cold queue, each candidate must pass parsability, solvability, and dataset-format filters.
Accepted candidates are then inserted as cold records and receive empirical rollout-derived pool statistics only after their first verified rollout group.

\subsection{Teacher Filtering and Manual Audit}
During training, every policy-generated candidate query is sent to the asynchronous teacher for final-answer annotation.
We log the total number of generated candidates, acceptance rate, rejection reasons, and a manual audit of accepted candidates.
The audit checks whether the teacher-provided answer is consistent with the generated problem statement and verifier format.
Table~\ref{tab:teacher_quality} summarizes the results.

\begin{table*}[t]
\centering
\resizebox{\textwidth}{!}{
\begin{tabular}{@{}lrr|rrr|rr@{}}
\toprule
\multirow{2}{*}{Training Corpus}
& \multicolumn{2}{c|}{Candidate Flow}
& \multicolumn{3}{c|}{Rejected Candidates}
& \multicolumn{2}{c}{Manual Audit} \\
\cmidrule(lr){2-3} \cmidrule(lr){4-6} \cmidrule(lr){7-8}
& \# Candidates & Accepted (\%)
& Unparsable (\%) & Ambig./Unsolv. (\%) & Format (\%)
& Audit Size & Wrong Label (\%) \\
\midrule
\textsc{DAPO-Math} & 44K  & 71.3 & 2.3 & 23.1 & 3.3 & 200 & 1.5 \\
\textsc{OpenR1-Math} & 237K & 68.5 & 0.8 & 28.6 & 2.1 & 200 & 0.8 \\
\midrule
Overall & 281K & 68.9 & 1.0 & 27.7 & 2.3 & 400 & 1.2 \\
\bottomrule
\end{tabular}
}
\caption{
Teacher filtering and manual audit results for policy-generated candidate queries using GPT-5-nano as the asynchronous answer annotator.
Rejected candidates are grouped by rejection reason: unparsable final-answer format, ambiguous or unsolvable problem statement, and dataset/verifier-format violation.
Wrong Label denotes the estimated wrong-label rate among accepted candidates from manual audit.
}
\label{tab:teacher_quality}
\end{table*}

The results show that a substantial fraction of generated candidates is filtered before entering training.
The largest rejection category is ambiguous or unsolvable candidates, reflecting the difficulty of using small policy models as query proposers.
After filtering, the estimated wrong-label rate among accepted candidates is low: 1.5\% on \textsc{DAPO-Math}, 0.8\% on \textsc{OpenR1-Math}, and 1.2\% overall.

\subsection{Synthetic Label-Noise Robustness}
We further test whether \name{} is robust to residual accepted-label noise.
Starting from accepted augmented queries, we synthetically corrupt a controlled fraction of their final answers.
The corruption is applied only to augmented queries, while seed-query labels remain unchanged.
We then train under the same GRPO recipe and evaluate Avg@16 on the same seven benchmarks.
This isolates the effect of wrong teacher labels from the other components of \name{}.

\begin{table*}[t]
\centering
\resizebox{0.85\textwidth}{!}{
\begin{tabular}{@{}lllrr@{}}
\toprule
Method & Training Dataset & Optimizer & Synthetic Noise Rate & Avg@16 \\
\midrule
\name{} & \textsc{DAPO-Math} & GRPO & 0\% & 51.1 \\
\name{} & \textsc{DAPO-Math} & GRPO & 1\% & 50.0 \\
\name{} & \textsc{DAPO-Math} & GRPO & 3\% & 49.4 \\
\name{} & \textsc{DAPO-Math} & GRPO & 5\% & 48.2 \\
\midrule
\name{} & \textsc{OpenR1-Math} & GRPO & 0\% & 44.0 \\
\name{} & \textsc{OpenR1-Math} & GRPO & 1\% & 43.0 \\
\name{} & \textsc{OpenR1-Math} & GRPO & 3\% & 42.5 \\
\name{} & \textsc{OpenR1-Math} & GRPO & 5\% & 41.6 \\
\bottomrule
\end{tabular}
}
\caption{
Robustness to synthetic corruption of accepted augmented-query labels.
Noise is injected only into accepted augmented labels, while seed labels are unchanged.
}
\label{tab:teacher_noise_robustness}
\end{table*}

The results show that \name{} is not immune to label noise, but its degradation is gradual under moderate synthetic corruption.
Since the measured accepted-label error rate after filtering is 1.2\% overall, the observed teacher quality is close to the lowest tested noise regime, where the method remains stable.

\section{Prompts Used for Augmentation and Verification}
In this section, we first document the prompt templates used by \name{} for (i) student rollouts, (ii) on-policy query augmentation, and (iii) asynchronous teacher verification.
We include the exact templates because small formatting differences can affect strict answer parsing, candidate acceptance rates, and the distribution of augmented data.

\subsection{Policy Rollout Prompt}
\label{app:prompts:solver}

\begin{tcolorbox}[promptbox,title={Student rollout prompt}]
\begin{PromptVerbatim}
<|im_start|>system
You are Qwen, created by Alibaba Cloud. You are a helpful assistant.<|im_end|>
<|im_start|>user
Solve the following math problem step by step.
The last line of your response should be of the form Answer: $Answer (without quotes)
where $Answer is the answer to the problem.

\textcolor{PromptIn}{<PROBLEM_STATEMENT>}

Remember to put your answer on its own line after "Answer:".<|im_end|>
<|im_start|>assistant
\end{PromptVerbatim}
\end{tcolorbox}

We use a fixed rollout instruction template (following the standard evaluation/training protocol of~\citet{DAPO}) and substitute
\textcolor{PromptIn}{\texttt{<PROBLEM\_STATEMENT>}} with the problem text $x$ when sampling rollouts $y\sim\pi_\theta(\cdot\mid x)$.
This template serves two practical purposes.
First, it normalizes the interaction format across models and methods (system/user roles and instruction wording), so the only changing input is the problem content.
Second, it enforces a \emph{deterministic answer delimiter}: the model is required to place the final answer on a separate last line beginning with \texttt{Answer:}.
Our verifier $R(\cdot)$ extracts the numeric string from this last line and applies strict exact matching against the ground-truth answer $g(x)$.
Keeping the rollout prompt identical across all ablations and baselines ensures that any performance differences are attributable to the query-lifecycle mechanisms (sampling, augmentation, refresh, reinsertion), rather than to formatting advantages or parsing artifacts.

\subsection{On-policy Augmentation Prompt}
\label{app:prompts:augment}

\begin{tcolorbox}[promptbox,title={On-policy augmentation prompt}]
\begin{PromptVerbatim}
You are an expert math problem writer.
Given an original problem, produce a NEW problem that is similar by changing ONLY 
numeric values (constants, coefficients, exponents, lengths, bounds) by small amounts. 
Keep topic/structure/variables/target the same.
Do NOT output any code or solutions to the problem. 
Do not generate any rationale or explanations.
Follow this format exactly, with no extra text.

\textcolor{PromptTag}{<ORIG>}
\textcolor{PromptIn}{<ORIGINAL_PROBLEM_TEXT_ONLY>}

\textcolor{PromptTag}{<NEW>}
\textcolor{PromptOut}{<NEW_PROBLEM_TEXT_ONLY>}

\textcolor{PromptTag}{<DIFF>}
\textcolor{PromptOut}{<NUMBER_IN_[0.75,1.33]>}

\textcolor{PromptTag}{<END>}
\end{PromptVerbatim}
\end{tcolorbox}

We apply this prompt to the \emph{current} policy $\pi_\theta$ to generate candidate new queries on-policy.
Given a sampled parent query $x$, we feed the problem statement (without the answer) into the \texttt{<ORIG>} field and instruct the model to produce a minimally edited variant in the \texttt{<NEW>} field.
The key constraint---``change ONLY numeric values''---is intentional: it preserves the underlying topic, symbolic structure, and target quantity, while introducing controlled novelty via small perturbations of constants, coefficients, bounds, and exponents.
This greatly reduces the rate of ill-posed generations (e.g., missing variables, contradictory constraints, or multi-question outputs) compared to free-form rewriting, while still producing new instances that can be frontier-aligned as the policy improves.

The strict tagged output format serves as a low-cost interface between generation and verification.
We parse the \textcolor{PromptOut}{\texttt{<NEW\_PROBLEM\_TEXT\_ONLY>}} span as the candidate augmented query $x^{\mathrm{aug}}$ and discard any outputs that violate the tag structure.
The scalar in \texttt{<DIFF>} is recorded as the \emph{policy-estimated relative difficulty} $d_{\mathrm{pol}}$ for $x^{\mathrm{aug}}$.
Intuitively, $d_{\mathrm{pol}}>1$ indicates that the generated problem is estimated harder than its parent and $d_{\mathrm{pol}}<1$ indicates easier.
We clamp $d_{\mathrm{pol}}$ to $[0.75,1.33]$ and treat it as auxiliary metadata; in particular, it can be used by the pool statistics refresh module to downweight contributions from children that the policy itself believes are harder, without requiring any extra supervision from the teacher~\cite{wang2023cat,wang2023car,wang2024candle,wang2025conceptualizationsurvey,DBLP:conf/acl/WangCLNX0SLGLYB25,wang2026arxiv2table,wang2025mars}.

\subsection{Teacher Verification Prompt}
\label{app:prompts:teacher}

\begin{tcolorbox}[promptbox,title={Teacher verification prompt (answer-only JSON)}]
\begin{PromptVerbatim}
You will be given ORIGINAL and GENERATION.

Decide whether GENERATION contains a single, self-contained, well-posed math problem 
whose final answer is a number. When judging well-posedness, 
ignore superficial wrappers such as "Question:", "Assistant:", code fences,
or a trailing "Answer:" line. You do NOT need to output any cleaned question text.

If solvable, compute ONLY the final numeric answer as a bare number string.

Return ONLY one JSON object on a single line:
\textcolor{PromptOut}{\{"solvable":true|false,"answer":"<string or null>"\}}
Rules: lowercase true/false/null
"answer" must be a bare number string when solvable, else null.

ORIGINAL:
\textcolor{PromptIn}{<ORIGINAL>}

GENERATION:
\textcolor{PromptIn}{<GENERATION>}
\end{PromptVerbatim}
\end{tcolorbox}

We use a strong teacher model $\mathcal{T}$ as an \emph{external answer annotator} for policy-generated candidates.
For each candidate query $x^{\mathrm{aug}}$ produced by the augmentation prompt, we asynchronously call $\mathcal{T}$ with the parent \texttt{ORIGINAL} and the raw candidate text \texttt{GENERATION}.
The teacher is asked to perform only two decisions: (i) whether the candidate corresponds to a single well-posed math problem with a numeric final answer, and (ii) if so, what that final numeric answer is.
To keep the interface robust to minor formatting artifacts introduced by the policy, we explicitly allow the teacher to ignore superficial wrappers such as \texttt{Question:} prefixes, role markers (e.g., \texttt{Assistant:}), Markdown code fences, or trailing \texttt{Answer:} lines.
Crucially, the teacher is \emph{not} asked to rewrite or ``clean'' the prompt text, and it returns \emph{only} a single-line JSON object.

After receiving the JSON response, we apply strict acceptance rules.
A candidate is accepted iff \texttt{solvable=true} and the \texttt{answer} field can be parsed as a bare numeric string.
Accepted candidates are converted into new verifiable training pairs $(x^{\mathrm{aug}}, g(x^{\mathrm{aug}}))$ with $g(x^{\mathrm{aug}})=\texttt{answer}$.
The resulting record is inserted into the pool as a cold (unscored) item, ensuring that it is sampled early to obtain a first rollout-based success statistic.
Candidates that are unsolvable, ambiguous, or yield non-numeric answers are discarded.
Because verification is performed asynchronously, the main RLVR loop does not block on teacher latency: verified results are drained whenever available and incrementally grow the pool over time.
Finally, any lightweight deterministic stripping needed for tokenization (e.g., removing wrappers or trailing \texttt{Answer:} blocks) is performed on our side before insertion; this keeps the teacher contract minimal and avoids coupling data quality to teacher rewriting behavior~\cite{yu2023folkscope,ding2024intentionqa,yang2026sessionintentbench,xu2024mind,shi2023qadynamics,DBLP:conf/emnlp/ZhengDTWBWS25,DBLP:conf/eacl/ChanCWJFLS24}.

\begin{table}[t]
\centering
\resizebox{\linewidth}{!}{
\begin{tabular}{@{}l|rrrrrrr@{}}
\toprule
Benchmark & AIME24 & AIME25 & AMC23 & GPQA & MATH500 & MinervaMath & OlympiadBench \\
\midrule
\#Entries & 30 & 30 & 40 & 448 & 500 & 272 & 150 \\
\bottomrule
\end{tabular}
}
\caption{Number of examples in each benchmark dataset.}
\label{tab:benchmark-sizes}
\end{table}

\section{Pool Statistics Propagation with the Lineage Graph}
\label{sec:analysis:propagation}
When on-policy augmentations are produced by small edits, many newly inserted queries are correlated through shared parents and templates.
If we reinsert archived records using only their most recent empirical success statistic, a burst of correlated children can temporarily distort sampling priorities and induce oscillations near the low/high boundary.
We therefore maintain an augmentation lineage graph and periodically \emph{refresh} each node's pool statistic by propagating signals from its descendants.

\textbf{Graph and notation.}
Let $\mathcal{G}=(\mathcal{V},\mathcal{E})$ be the lineage graph, where each node $q\in\mathcal{V}$ is a query record and each directed edge $(q \to c)\in\mathcal{E}$ indicates that $c$ is a child augmentation of $q$.\footnote{In implementation we store parent-to-children adjacency.
All results below are stated for a rooted tree for clarity; the path-count weighting extends directly to DAGs by replacing ``\#leaves'' with ``\#root-to-leaf paths'' as described below.}
Write $\mathcal{C}(q)$ for the children of $q$.
Let $\tilde r_q$ denote the stored pool statistic for $q$ (e.g., group pass rate or its refreshed value), and let $d_c\in[d_{\min},d_{\max}]$
be the \emph{policy-estimated relative difficulty} attached to child $c$ at augmentation time (the \texttt{<DIFF>} field in Appendix~\ref{app:prompts:augment}); for seeds or missing metadata we set $d_c\!\leftarrow\!1$.
We use a generic difficulty-adjusted child signal
\begin{equation}
\phi(\tilde r_c,d_c)\in[-1,1],
\label{eq:phi_def}
\end{equation}
which in our implementation is a sign-aware scaling followed by clipping to $[-1,1]$ (see \S\ref{sec:method:refresh}).

\textbf{Two propagation modes.}
We refresh parent scores by a convex blend of their current value and an aggregate of descendant signals:
\begin{equation}
\tilde r_q \leftarrow (1-\lambda)\,\tilde r_q + \lambda \cdot \mathrm{Agg}\Big(\{\,\phi(\tilde r_c,d_c)\,\}_{c\in \mathcal{C}(q)}\Big),
\qquad \lambda\in[0,1].
\label{eq:refresh_generic}
\end{equation}
We consider two choices of $\mathrm{Agg}(\cdot)$.

\smallskip
\noindent\textbf{(i) Children aggregation (ChildAgg).}
ChildAgg uses an unweighted mean across immediate children:
\begin{equation}
\mathrm{Agg}_{\textsc{Child}}(q)
= \frac{1}{|\mathcal{C}(q)|}\sum_{c\in\mathcal{C}(q)} \phi(\tilde r_c,d_c).
\label{eq:childagg}
\end{equation}
This treats each child as one ``vote,'' regardless of how large the child's subtree is.

\smallskip
\noindent\textbf{(ii) Path aggregation (PathAgg).}
PathAgg upweights children with larger subtrees, reflecting that a child with many descendants represents more downstream queries and thus
more evidence about the parent's region of the space.
Let $s(q)$ denote the number of leaf paths in the subtree of $q$.\footnote{For a tree, $s(q)$ is the number of leaves under $q$.
For a DAG, we use the number of root-to-leaf paths passing through $q$, computed bottom-up; this matches our implementation's leaf-count
logic and naturally handles shared descendants.}
Define normalized weights
\begin{equation}
w_{q\to c} \triangleq \frac{s(c)}{\sum_{c'\in\mathcal{C}(q)} s(c')}, \qquad \sum_{c\in\mathcal{C}(q)} w_{q\to c}=1,
\label{eq:path_weights}
\end{equation}
and set
\begin{equation}
\mathrm{Agg}_{\textsc{Path}}(q)
= \sum_{c\in\mathcal{C}(q)} w_{q\to c}\, \phi(\tilde r_c,d_c).
\label{eq:pathagg}
\end{equation}
We compute $s(\cdot)$ bottom-up once per refresh and then apply \eqref{eq:pathagg} in a bottom-up pass over depths.

\subsection{Why PathAgg is the Right Weighting}
\label{sec:analysis:pathagg_optimality}

PathAgg can be justified from two complementary viewpoints: (a) it estimates the mean descendant signal under a natural ``uniform-over-leaves''
(or ``uniform-over-paths'' in DAGs) objective, whereas ChildAgg estimates a different quantity that can be biased toward small subtrees; and
(b) under a simple noise model, PathAgg minimizes variance among unbiased aggregators.

\begin{proposition}[Unbiased estimation of the uniform-descendant objective]
\label{prop:unbiased}
Assume the lineage structure under a node $q$ is a rooted tree.
For each leaf $\ell$ under $q$, define its (difficulty-adjusted) leaf signal as
$z_\ell \in \mathbb{R}$ (e.g., $z_\ell = \phi(r_\ell,d_\ell)$ or a path-product variant).
Let $\mathcal{L}(q)$ be the set of leaves under $q$ and define the uniform-descendant objective
\begin{equation}
\mu(q) \triangleq \frac{1}{|\mathcal{L}(q)|}\sum_{\ell\in\mathcal{L}(q)} z_\ell.
\label{eq:uniform_leaf_mean}
\end{equation}
Partition leaves by the first edge from $q$: for each child $c\in\mathcal{C}(q)$, let $\mathcal{L}(c)$ be the leaves under $c$ and
$\mu(c)=\frac{1}{|\mathcal{L}(c)|}\sum_{\ell\in\mathcal{L}(c)} z_\ell$.
Then $\mu(q)$ decomposes as
\begin{equation}
\mu(q) = \sum_{c\in\mathcal{C}(q)} \frac{|\mathcal{L}(c)|}{|\mathcal{L}(q)|}\,\mu(c),
\label{eq:leaf_decompose}
\end{equation}
i.e., the unique unbiased combination of child means uses weights proportional to subtree sizes.
\end{proposition}

\begin{proof}
Each leaf under $q$ belongs to exactly one child subtree, so
\[
\sum_{\ell\in\mathcal{L}(q)} z_\ell
= \sum_{c\in\mathcal{C}(q)} \sum_{\ell\in\mathcal{L}(c)} z_\ell
= \sum_{c\in\mathcal{C}(q)} |\mathcal{L}(c)| \cdot \mu(c).
\]
Dividing both sides by $|\mathcal{L}(q)|=\sum_{c}|\mathcal{L}(c)|$ yields \eqref{eq:leaf_decompose}.
\end{proof}

Proposition~\ref{prop:unbiased} makes the core issue explicit:
ChildAgg in \eqref{eq:childagg} corresponds to weighting each child equally, which equals \eqref{eq:leaf_decompose} only when all child subtrees
have equal size.
When subtree sizes are skewed (a common outcome under on-policy augmentation where some templates are repeatedly edited), ChildAgg can systematically
underweight the larger subtree and overweight smaller ones.
PathAgg in \eqref{eq:pathagg} uses exactly the weights in \eqref{eq:leaf_decompose} (or their DAG analogue using path counts), and therefore matches
the uniform-descendant objective.

\begin{proposition}[Minimum-variance weighting under a simple noise model]
\label{prop:minvar}
Under the same setup as Proposition~\ref{prop:unbiased}, suppose each leaf signal is
$z_\ell = \mu_\ell + \varepsilon_\ell$ with independent noise $\varepsilon_\ell$ of zero mean and common variance $\sigma^2$.
Let $\widehat{\mu}(c)$ be the empirical mean of leaves under child $c$ (so $\mathrm{Var}[\widehat{\mu}(c)]=\sigma^2/|\mathcal{L}(c)|$).
Among all unbiased estimators of $\mu(q)$ of the form $\widehat{\mu}(q)=\sum_{c}\alpha_c \widehat{\mu}(c)$ with $\sum_c \alpha_c = 1$,
the choice $\alpha_c = |\mathcal{L}(c)|/|\mathcal{L}(q)|$ uniquely minimizes $\mathrm{Var}[\widehat{\mu}(q)]$.
\end{proposition}

\begin{proof}
Unbiasedness requires $\sum_c \alpha_c = 1$ and targets $\mu(q)$ as a convex combination of child means.
The variance is
\[
\mathrm{Var}\!\left[\sum_c \alpha_c \widehat{\mu}(c)\right]
= \sum_c \alpha_c^2 \,\mathrm{Var}[\widehat{\mu}(c)]
= \sigma^2 \sum_c \frac{\alpha_c^2}{|\mathcal{L}(c)|}.
\]
Minimizing $\sum_c \alpha_c^2/|\mathcal{L}(c)|$ subject to $\sum_c \alpha_c = 1$ via a Lagrange multiplier yields
$\alpha_c \propto |\mathcal{L}(c)|$, i.e., $\alpha_c = |\mathcal{L}(c)|/|\mathcal{L}(q)|$.
\end{proof}

Proposition~\ref{prop:minvar} gives a stronger statement than ``unbiased'':
when each child mean aggregates a different amount of evidence (more descendants $\Rightarrow$ lower estimation variance),
the optimal way to combine them is to weight by subtree size.
This is precisely what PathAgg implements via $w_{q\to c}$.

\subsection{Connection to Our Refresh Update}
\label{sec:analysis:refresh_connection}
Our refresh applies \eqref{eq:refresh_generic} recursively with $\lambda\in[0,1]$ (we use $\lambda=0.5$).
The recursion can be unrolled to show that the refreshed score is a geometrically discounted mixture of descendant signals along paths.
For a tree, define a random process that starts at $q$ and at each internal node transitions to a child $c$ with probability $w_{q\to c}$.
Let $P(q\!\to\!v)$ denote the set of nodes $v$ reachable from $q$ at depth $\ell$.
Then repeated bottom-up PathAgg updates induce contributions proportional to (i) the product of path weights (which makes leaves approximately uniform
under the path-count definition), (ii) the difficulty adjustments along the path through $\phi(\cdot)$, and (iii) a depth discount controlled by $\lambda$.
Intuitively, PathAgg therefore tracks an ``average downstream learnability'' of a query under the same subtree structure that generated its augmentations,
whereas ChildAgg treats each immediate child equally and can amplify noise from small subtrees or transient branches.

\textbf{Practical implication.}
In our setting, augmentation can be highly unbalanced: some parents generate many valid children (large $s(c)$), while others yield only a few.
Because the pool is capacity-bounded and sampling focuses near the boundary, mis-estimating a parent's statistic can have outsized effects on which region
is repeatedly sampled.
PathAgg uses more of the evidence where it exists (large subtrees) and downweights fragile branches, which leads to a smoother, more stable boundary signal.
This provides a plausible mechanism for the empirical result PathAgg $>$ ChildAgg observed across datasets and optimizers.

\begin{figure}[t]
    \centering
    \includegraphics[width=0.8\linewidth]{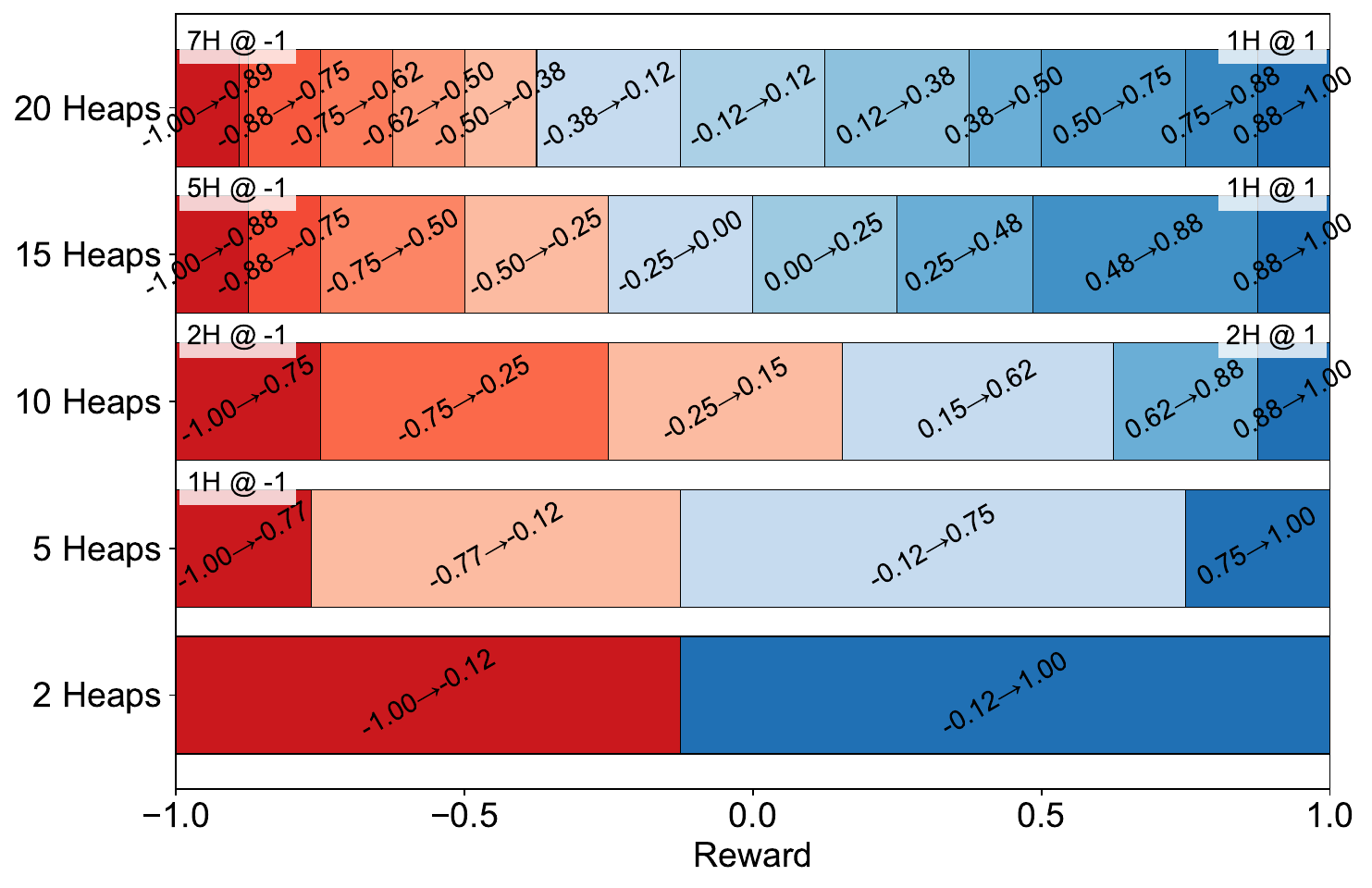}
    \caption{
    Visualization of reward-axis partitioning under multi-heap prompt pools.
    Each row shows the contiguous reward intervals assigned to bins (heaps) for a given $H$ (we linearly rescale the pool statistic $\tilde r\in[-1,1]$ to $[-1,1]$ for visualization, where larger values indicate easier prompts).
    The numbers inside each segment denote the corresponding interval endpoints.
    Labels of the form ``$k$H @ $-1$'' indicate that $k$ bins are \emph{anchored and concentrated near the hardest end} of the reward axis (starting at $-1$), while ``$k$H @ $1$'' indicates bins anchored near the easiest end (ending at $1$).
    As $H$ increases, more boundaries fall into the extreme-hard region, which increases the chance that boundary-based sampling selects too-difficult prompts.
    }
    \label{figure:multi-heap-boundary}
\end{figure}

\begin{table*}[t]
\centering
\resizebox{\textwidth}{!}{%
\begin{tabular}{@{}lp{2.8cm}|ccccccc|c@{}}
\toprule
Train Set & \#Heaps & AIME24 & AIME25 & AMC23 & GPQA & MATH500 & MinervaMath & Olym.Bench & Avg. \\
\midrule
\multicolumn{10}{l}{\emph{Training Dataset: \textsc{DAPO-Math}}~\citep{DAPO}} \\
\midrule
\textsc{DAPO-Math} & 2  & 21.4 & 21.0 & 82.4 & 42.7 & 85.6 & 51.4 & 52.9 & 51.1 \\
\textsc{DAPO-Math} & 5  & 16.7 & 21.5 & 81.6 & 38.6 & 80.1 & 42.2 & 48.6 & 42.8 \\
\textsc{DAPO-Math} & 10 & 21.7 & 24.8 & 80.5 & 38.0 & 81.5 & 42.1 & 49.0 & 43.9 \\
\textsc{DAPO-Math} & 15 & 13.3 & 9.2 & 77.0 & 33.1 & 59.5 & 23.7 & 30.3 & 30.9 \\
\textsc{DAPO-Math} & 20 & 12.1 & 11.7 & 75.3 & 32.2 & 55.7 & 18.8 & 28.7 & 29.2 \\
\midrule
\multicolumn{10}{l}{\emph{Training Dataset: \textsc{OpenR1-Math}}~\citep{OpenR1}} \\
\midrule
\textsc{OpenR1-Math} & 2  & 20.6 & 15.5 & 68.4 & 38.5 & 81.3 & 37.7 & 45.9 & 44.0 \\
\textsc{OpenR1-Math} & 5  & 12.9 & 16.0 & 63.6 & 36.5 & 78.5 & 37.5 & 45.1 & 37.2 \\
\textsc{OpenR1-Math} & 10 & 16.0 & 14.6 & 61.7 & 36.5 & 78.5 & 35.2 & 45.0 & 36.8 \\
\textsc{OpenR1-Math} & 15 & 14.2 & 9.8 & 60.0 & 35.0 & 75.6 & 25.5 & 40.7 & 33.0 \\
\textsc{OpenR1-Math} & 20 & 15.2 & 11.9 & 57.8 & 34.1 & 75.3 & 27.6 & 41.6 & 33.4 \\
\bottomrule
\end{tabular}%
}
\caption{Effect of the number of heaps for GRPO with Qwen2.5-7B-Instruct on \textsc{DAPO-Math} and \textsc{OpenR1-Math}.}
\label{table:heap-ablation}
\end{table*}

\section{Fine-Graining Difficulty Selection with Multi-Heaps}
\label{app:multi_heaps}
A core motivation of \name{} is to track the \emph{moving capability frontier}: prompts that are neither always solved nor always failed under the current policy.
In group-based RLVR, such medium-difficulty prompts are particularly valuable because they tend to produce mixed-success rollout groups, yielding informative within-prompt advantages.
Our default two-heap pool implements this idea by maintaining a single boundary between ``hard'' and ``easy'' scored items and sampling from a narrow band around that boundary.

This naturally raises a question: can we obtain an even more precise curriculum by using \emph{more than two heaps}?
Intuitively, increasing the number of heaps $H$ partitions the difficulty axis into finer bins and creates more local boundaries, potentially enabling the sampler to focus on a narrower difficulty slice.
To explore this hypothesis, we generalize the two-heap pool into a multi-heap pool and sample queries from the gaps between adjacent heaps (i.e., the boundaries between consecutive bins), which can be viewed as a multi-resolution extension of boundary sampling.

\paragraph{Multi-heap pool and boundary sampling.}
Recall that each query record maintains a scalar pool statistic $\tilde r\in[-1,1]$ (e.g., group pass rate or its refreshed variant), used only for sampling.
To visualize the partitioning, we linearly rescale $\tilde r$ to the interval $[-1,1]$ (higher $\Rightarrow$ easier) and show the induced bin boundaries in
Figure~\ref{figure:multi-heap-boundary}.
For $H{=}2$, the pool is split into two coarse partitions separated at a single boundary (e.g., $[-1.00,-0.12]$ vs.\ $[-0.12,1.00]$; see the first row in the figure).
For larger $H$, we partition the reward axis into more bins and sample queries around the \emph{gaps} (i.e., boundaries) between adjacent heaps, with the intent
to concentrate probability mass on medium difficulty regions at higher resolution.

\paragraph{Why increasing the number of heaps can hurt.}
Empirically, $H{=}2$ achieves the strongest average results on both training corpora, while larger $H$ quickly degrades (Table~\ref{table:heap-ablation}).
We attribute this to a mismatch between (i) the \emph{geometry of the reward distribution} in RLVR pools and (ii) the \emph{uniform treatment of boundaries}
implicit in multi-heap sampling.

\noindent\textbf{(1) More heaps create more boundaries in the extreme-hard tail.}
In RLVR math pools, the scored-statistic distribution is typically heavy near the low end early-to-mid training (many prompts are still mostly unsolved).
As a result, when we increase $H$, a large fraction of the additional bins end up splitting the already-dense hard tail, creating \emph{many} adjacent boundaries
whose neighborhoods correspond to very low success.
This is exactly what Figure~\ref{figure:multi-heap-boundary} shows: as $H$ increases, the partition introduces multiple narrow intervals near $-1$
(e.g., the 15-heap configuration allocates ``5H @ -1'' and the 20-heap configuration allocates ``7H @ -1'', meaning many bins are concentrated at the hardest end).
Consequently, sampling ``around boundaries between heaps'' increasingly selects queries from the too-hard region.

\noindent\textbf{(2) Too-hard samples collapse group advantages and waste rollouts.}
Group-based RLVR methods (GRPO/DAPO) learn most from prompts that produce \emph{mixed} rollout outcomes (some correct, some incorrect), because the within-prompt
baseline yields non-degenerate advantages.
Extremely hard bins instead produce nearly all-zero rewards in a group, leading to near-zero advantages and weak gradient signal while consuming the same rollout budget.
Thus, over-selecting boundaries inside the hard tail directly reduces sample efficiency and can slow or even destabilize training.

\noindent\textbf{(3) Granularity increases noise and churn under fixed pool capacity.}
With fixed pool capacity $N$, increasing $H$ reduces the expected number of items per bin, which makes per-bin boundary statistics noisier and more sensitive to
stochastic pass-rate fluctuations, recent refresh values, and eviction/reinsertion dynamics.
This amplifies ``bin hopping'' (records frequently crossing nearby thresholds) and can create curriculum oscillations: the sampler overreacts to small score changes
because the boundaries are dense and each bin contains fewer items.

\noindent\textbf{(4) Tail-biased partitioning is structural in our multi-heap design.}
Our multi-heap configurations intentionally allocate some number of bins anchored at the extremes to preserve resolution at the tails (notations like
``$k$H @ $-1$'' and ``$k$H @ $1$'').
However, because the hard tail is typically much more populated than the easy tail during training, allocating many anchored bins at $-1$ effectively increases
the number of hard-tail boundaries the sampler can land on.
For instance, the figure explicitly indicates settings such as ``2H @ -1 2H @ 1'' for $H{=}10$ and ``7H @ -1 1H @ 1'' for $H{=}20$, illustrating that
a large portion of additional resolution is spent near the hardest region.
Under boundary-based sampling, this inevitably increases exposure to extreme difficulty, which aligns with the performance degradation observed at $H{\ge}15$.

\begin{figure*}[t]
    \centering
    \begin{subfigure}[t]{0.49\textwidth}
        \centering
        \includegraphics[width=\linewidth]{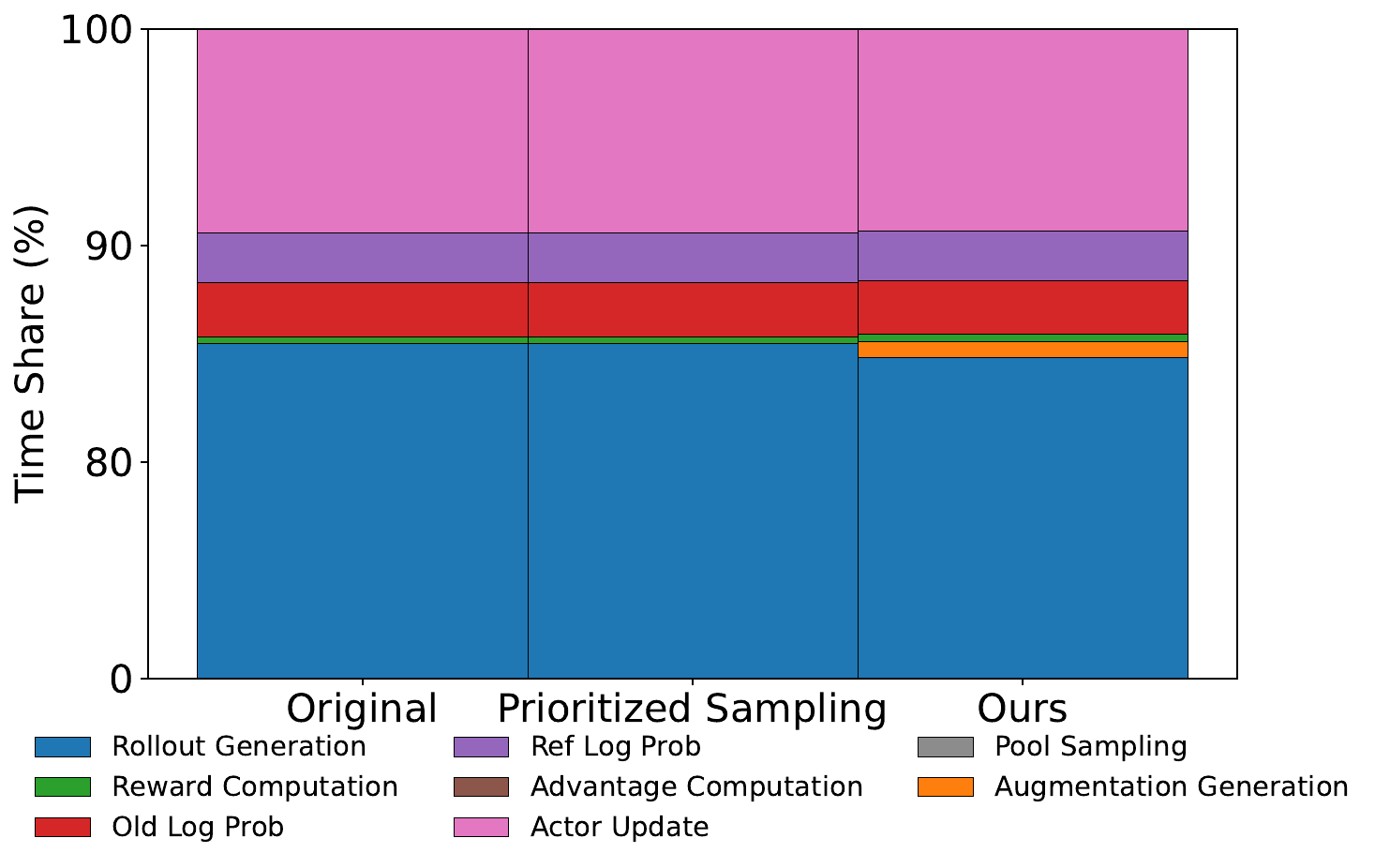}
        \caption{
        Stage-wise time share for the 8B model under three training configurations
        (Original, Prioritized Sampling, and Ours). The y-axis is truncated to
        highlight the contribution of non-rollout stages.
        }
        \label{fig:time_breakdown_8b}
    \end{subfigure}\hfill
    \begin{subfigure}[t]{0.49\textwidth}
        \centering
        \includegraphics[width=\linewidth]{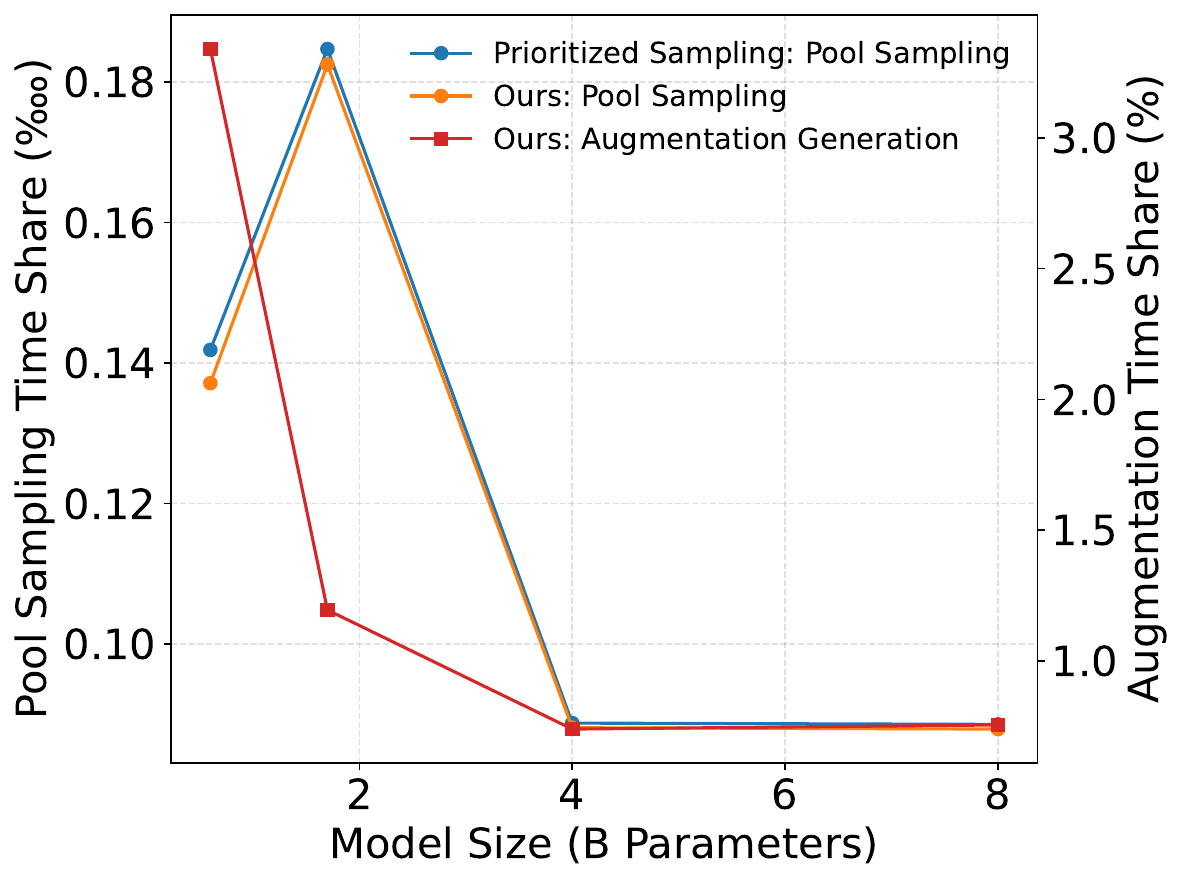}
        \caption{
        Pool sampling and augmentation time share versus model size.
        The left y-axis shows pool sampling time share in per-10,000 units of total time, and the right y-axis shows augmentation time
        share in percent for our method.
        }
        \label{fig:time_share_vs_scale}
    \end{subfigure}

    \caption{
    Training-time overhead of \name{}.
    Left: stage-wise wall-clock time share for the 8B backbone (zoomed to emphasize non-rollout stages).
    Right: time share of pool sampling (per-10,000 units of total time) and on-policy augmentation generation (\% of total time) versus model size.
    }
    \label{fig:training_time_overhead_two_panel}
\end{figure*}

\section{Training-Time Breakdown and Overhead}
\label{app:time_breakdown}
A potential concern for \name{} is that we introduce additional mechanisms beyond a standard RLVR loop: (i) a more complex sampler (heap-based boundary sampling) and
(ii) an explicit on-policy query augmentation budget.
Both add extra bookkeeping and generation calls, so it is important to verify that the method does not materially degrade wall-clock efficiency.

\paragraph{Profiling protocol.}
We measure per-step wall-clock time and decompose it into the stages shown in Figure~\ref{fig:time_breakdown_8b}.
The dominant stage in RLVR is rollout generation (sampling $n$ long responses per prompt), and we therefore visualize stage-wise shares with the y-axis truncated
to highlight small but potentially non-negligible components.
Teacher verification runs asynchronously and is not on the critical path; the overhead shown for \name{} includes (a) the pool sampling logic executed on the driver
and (b) the student-side augmentation generation that produces candidate problems before asynchronous verification.

\paragraph{Stage-wise time share at 8B: rollout dominates, and \name{} adds only a thin slice.}
Figure~\ref{fig:time_breakdown_8b} shows the breakdown for the 8B model under three configurations (Original, Prioritized Sampling, and Ours).
As expected, rollout generation occupies the vast majority of time, while reward computation, log-prob extraction, advantage computation, and the actor update
share the remaining small fraction.
Within this zoomed view, the additional stages introduced by \name{}---\emph{pool sampling} and \emph{augmentation generation}---are visible but remain minor
relative to rollout generation and the standard RLVR bookkeeping.
This supports the design goal that query-lifecycle mechanisms should primarily \emph{reallocate} rollout budget toward more learnable prompts rather than
becoming a new bottleneck.

\paragraph{Scaling trend: pool sampling is essentially free, and augmentation share stays small and shrinks with model size.}
Figure~\ref{fig:time_share_vs_scale} isolates the two new components as model size increases.
Pool sampling time share is plotted in per-10,000 units of total time and stays around $0.10$--$0.18$ across sizes for both prioritized sampling and our method,
i.e., well below $0.002\%$ of total wall-clock time.
This is expected because boundary sampling consists of a small number of CPU-side heap operations (amortized $O(\log N)$ per pop/push with modest constants),
performed once per training step.

Augmentation generation is plotted on the right y-axis (percent) and remains in the low single-digit range (roughly $1$--$3\%$ in the figure).
Moreover, its fraction decreases as model size grows.
The reason is that the augmentation stage generates comparatively short candidate prompts, while rollout generation scales steeply with model size and dominates
GPU time due to long responses and large rollout groups; as the rollout cost grows, a largely fixed augmentation budget becomes a smaller fraction of the step.

\section{\name{} vs.\ Distillation Pipelines}
\label{app:distillation}
Finally, we clarify why \name{} uses the \emph{student} policy $\pi_\theta$ to generate new queries and uses the teacher $\mathcal{T}$ only to annotate verifiable answers, rather than applying stronger teacher-driven distillation. We briefly contrast two common distillation pipelines.

\paragraph{(1) Teacher-generated question--answer distillation.}
A strong teacher can synthesize both questions and answers and the student imitates them via supervised training.
While effective for coverage expansion, this pipeline is \emph{teacher-centric}: the synthetic query distribution is shaped by the teacher and can drift away
from the student's \emph{moving capability frontier}. In RLVR, such drift often yields batches that are too easy or too hard for the current student, producing
near-saturated rollout groups and weak within-prompt advantages (Eq.~\eqref{eq:group-advantage}). In contrast, \name{} generates candidates on-policy from
$\pi_\theta$ (Appendix~\ref{app:prompts:augment}), which implicitly calibrates novelty and difficulty to the student's current state; the teacher is used only
to provide $g(x)$, keeping the optimization anchored to verifier-computed rewards.

\paragraph{(2) On-policy distillation (OPD).}
OPD supervises the student on its own rollouts using a teacher, e.g., via token-level logits or black-box teacher feedback on-policy~\citep{lu2025onpolicydistillation,blackboxOPD}.
However, OPD typically requires running the teacher over long rollout trajectories, so teacher compute scales with rollout length and can add substantial overhead
when rollout generation dominates. \name{} avoids this by using the teacher only for short, answer-only verification calls executed asynchronously
(\S\ref{sec:method:aug}), and by improving sample-efficiency through frontier-focused boundary sampling and pool growth rather than dense trajectory supervision.

\end{document}